\documentclass{article}

\usepackage{mkolar_definitions}
\usepackage{natbib,algorithm,algorithmic}

\usepackage{multirow}
\usepackage{rotating}

\makeatletter
\let\pen\@undefined
\newcommand{\pen}{\rho}
\makeatother

\newcommand{\supp}{{\rm supp}}
\newcommand{\uSc}{_{S^{C}}}
\newcommand{\acro}{HIPPO }
\newcommand{\Pconv}{\overset{\PP}{\rightarrow}}

\begin{document}

\title{Variance function estimation in high-dimensions}
\author{Mladen Kolar, James Sharpnack\footnote{Authors listed alphabetically.}}

\maketitle

\begin{abstract}
  We consider the high-dimensional heteroscedastic regression model,
  where the mean and the log variance are modeled as a linear
  combination of input variables. Existing literature on
  high-dimensional linear regression models has largely ignored
  non-constant error variances, even though they commonly occur in a
  variety of applications ranging from biostatistics to finance. In
  this paper we study a class of non-convex penalized pseudolikelihood
  estimators for both the mean and variance parameters. We show that
  the Heteroscedastic Iterative Penalized Pseudolikelihood Optimizer
  (HIPPO) achieves the oracle property, that is, we prove that the
  rates of convergence are the same as if the true model was known. We
  demonstrate numerical properties of the procedure on a simulation
  study and real world data.
\end{abstract}

\section{Introduction}

High-dimensional regression models have been studied extensively in
both machine learning and statistical literature. Statistical
inference in high-dimensions, where the sample size $n$ is smaller
than the ambient dimension $p$, is impossible without assumptions.  As
the concept of parsimony is important in many scientific domains, most
of the research in the area of high-dimensional statistical inference
is done under the assumption that the underlying model is sparse, in
the sense that the number of relevant parameters is much smaller than
$p$, or that it can be well approximated by a sparse model.

Penalization of the empirical loss by the $\ell_1$ norm has become a
popular tool for obtaining sparse models and huge amount of literature
exists on theoretical properties of estimation procedures
\citep[see,e.g.,][and references therein]{zhao06model,
  wainwright06sharp, zhang09some, zhang08sparsity} and on efficient
algorithms that numerically find estimates \citep[see][for an
extensive literature review]{bach11optimization}.  Due to limitations
of the $\ell_1$ norm penalization, high-dimensional inference methods
based on the class of concave penalties have been proposed that have
better theoretical and numerical properties
\citep[see,e.g.,][]{fan01variable, fan09nonconcave, lv09unified,
  zhang11general}. 

In all of the above cited work, the main focus is on model selection
and mean parameter estimation. Only few papers deal with estimation of
the variance in high-dimensions \citep{sun11scaled, fan12variance}
although it is a fundamental problem in statistics. Variance appears
in the confidence bounds on estimated regression coefficients and is
important for variable selection as it appears in Akaike's information
criterion (AIC) and the Bayesian information criterion
(BIC). Furthermore, it provides confidence on the predictive
performance of a forecaster.

In applied regression it is often the case that the error variance is
non-constant.  Although the assumption of a constant variance can
sometimes be achieved by transforming the dependent variable, e.g., by
using a Box-Cox transformation, in many cases transformation does not
produce a constant error variance \citep{carroll88transformation}.
Another approach is to ignore the heterogeneous variance and use standard
estimation techniques, but such estimators are less efficient. Aside from
its use in reweighting schemes, estimating variance is important because 
the resulting prediction intervals become more accurate and
it is often important to explore which input variables drive
the variance. In this paper, we will model the variance directly as a
parametric function of the explanatory variables. 

Heteroscedastic regression models are used in a variety of fields
ranging from biostatistics to econometrics, finance and quality
control in manufacturing. In this paper, we study penalized estimation
in high-dimensional heteroscedastic linear regression models, where
the mean and the log variance are modeled as a linear combination of
explanatory variables.  Modeling the log variance as a linear
combination of the explanatory variables is a common choice as it
guarantees positivity and is also capable of capturing variance that
may vary over several orders of magnitudes
\cite{carroll88transformation, harvey76}. Main contributions of this
paper are as follows. First, we propose HIPPO (Heteroscedastic
Iterative Penalized Pseudolikelihood Optimizer) for estimation of both
the mean and variance parameters. Second, we establish the oracle
property (in the sense of \citet{fan09nonconcave}) for the estimated
mean and variance parameters.  Finally, we demonstrate numerical
properties of the proposed procedure on a simulation study, where it
is shown that HIPPO outperforms other methods, and analyze a real data
set.

\subsection{Problem Setup and Notation}

Consider the usual heteroscedastic linear model,
\begin{equation}
  \label{eq:model}
  y_i = \xb_i' \betab + \sigma(\xb_i, \thetab) \epsilon_i, 
  \quad i = 1,\ldots,n,  
\end{equation}
where $\Xb = (\xb_1, \ldots, \xb_n)' = (\Xb_1, \ldots, \Xb_p)$ is an
$n \times p$ matrix of predictors with \iid~rows $\xb_1, \ldots,
\xb_n$, $\yb = (y_1, \ldots, y_n)$ is an $n$-vector of responses, the
vectors $\betab \in \RR^{p}$ and $\thetab \in \RR^{p}$ are $p$-vectors
of mean and variance parameters, respectively, and $\epsilonb =
(\epsilon_1, \ldots \epsilon_n)$ is an $n$-vector of \iid~random noise
with mean $0$ and variance $1$. We assume that the noise $\epsilonb$
is independent of the predictors $\Xb$. The function $\sigma(\xb,
\thetab)$ has a known parametric form and, for simplicity of
presentation, we assume that it takes a particular form $\sigma(\xb_i,
\thetab) = \exp(\xb_i'\thetab/2)$. 

Throughout the paper we use $[n]$ to denote the set $\{1,\ldots,n\}$.
For any index set $S \subseteq [p]$, we denote $\betab_S$ to be the
subvector containing the components of the vector $\betab$ indexed by
the set $S$, and $\Xb_S$ denotes the submatrix containing the columns
of $\Xb$ indexed by $S$. For a vector $\ab \in \RR^n$, we denote ${\rm
  supp}(\ab) = \{j\ :\ a_j \neq 0\}$ the support set, $\norm{\ab}_q$,
$q \in (0,\infty)$, the $\ell_q$-norm defined as $\norm{\ab}_q =
(\sum_{i\in[n]} a_i^q)^{1/q}$ with the usual extensions for $q \in
\{0,\infty\}$, that is, $\norm{\ab}_0 = |{\rm supp}(\ab)|$ and
$\norm{\ab}_\infty = \max_{i\in[n]}|a_i|$. For notational simplicity,
we denote $\norm{\cdot} = \norm{\cdot}_2$ the $\ell_2$ norm. For a
matrix $\Ab \in \RR^{n \times p}$ we denote $\opnorm{\Ab}{2}$ the
operator norm, $\norm{\Ab}_F$ the Frobenius norm, and
$\Lambda_{\min}(\Ab)$ and $\Lambda_{\max}(\Ab)$ denote the smallest
and largest eigenvalue respectively.

Under the model in \eqref{eq:model}, we are interested in estimating
both $\betab$ and $\thetab$. In high-dimensions, when $p \gg n$, it is
common to assume that the support $\betab$ is small, that is, $S =
{\rm supp}(\betab)$ and $|S| \ll n$. Similarly, we assume that the
support $T = {\rm supp}(\thetab)$ is small.

\subsection{Related Work}

Consider the model \eqref{eq:model} with constant variance, i.e.,
$\sigma(\xb, \thetab) \equiv \sigma_0$. Most of the existing
high-dimensional literature is focused on estimation of the mean
parameter $\betab$ in this homoscedastic regression model. Under a
variety of assumptions and regularity conditions, any penalized
estimation procedure mentioned in introduction can, in theory, select
the correct sparse model with probability tending to $1$. Literature
on variance estimation is not as developed. \citet{fan12variance}
proposed a two step procedure for estimation of the unknown variance
$\sigma_0$, while \cite{sun11scaled} proposed an estimation procedure
that jointly estimates the model and the variance.

Problem of estimation in the heteroscedastic linear regression models
have been studied extensively in the classical setting with $p$ fixed,
however, the problem of estimation under the model~\eqref{eq:model}
when $p \gg n$ has not been adequately studied.  \citet{jia10lasso}
assume that $\sigma(\xb, \thetab) = |\xb'\betab|$ and show that Lasso
is sign consistent for the mean parameter $\betab$ under certain
conditions. Their study shows limitations of lasso, for which many
highly scalable solvers exist. However, no new methodology is
developed, as the authors acknowledge that the log-likelihood function
is highly non-convex.  \citet{dette11adaptive} study the adaptive
lasso under the model in \eqref{eq:model}.  Under certain regularity
conditions, they show that the adaptive lasso is consistent, with
suboptimal asymptotic variance. However, the weighted adaptive lasso
is both consistent and achieves optimal asymptotic variance, under the
assumption that the variance function is consistently
estimated. However, they do not discuss how to obtain an estimator of
the variance function in a principled way and resort to an ad-hoc
fitting of the residuals.  \citet{daye11high} develop HHR procedure
that optimizes the penalized log-likelihood under \eqref{eq:model}
with the $\ell_1$-norm penalty on both the mean and variance
parameters. As the objective is not convex, HHR estimates $\betab$
with $\thetab$ fixed and then estimates $\thetab$ with $\betab$ fixed,
until convergence. Since the objective is biconvex, HHR converges to a
stationary point. However, no theory is provided for the final
estimates.

\section{Methodology}

In this paper, we propose HIPPO (Heteroscedastic Iterative Penalized
Pseudolikelihood Optimizer) for estimating $\betab$ and $\thetab$
under model~\eqref{eq:model}.

In the first step, HIPPO finds the penalized pseudolikelihood
maximizer of $\betab$ by solving the following objective
\begin{equation}
\label{eq:stage1}
\hat\betab = \arg
\min_{\betab \in \RR^{p}}
  \norm{\yb - \Xb\betab}^2 + 
  2 n \sum_{j \in [p]} \pen_{\lambda_S}(|\beta_j|),
\end{equation}
where $\pen_{\lambda_S}$ is the penalty function and the tuning
parameter $\lambda_S$ controls the sparsity of the solution
$\hat\betab$. 

In the second step, HIPPO forms the penalized pseudolikelihood
estimate for $\thetab$ by solving 
\begin{equation}
  \label{eq:stage2}
\begin{aligned}
  \hat\thetab = \arg \min_{\thetab \in \RR^{p}}
  \sum_{i \in [n]} \xb_i'\thetab & + 
  \sum_{i \in [n]} \hat{\eta}_i^2 \exp(-\xb_i'\thetab) \\
  & \quad + 4 n \sum_{j \in [p]} \rho_{\lambda_T}(|\theta_j|)
\end{aligned}
\end{equation}
where $\hat{\etab} = \yb - \Xb \hat{\betab}$ is the vector of
residuals.

Finally, HIPPO computes the reweighted estimator of the mean by
solving
\begin{equation}
  \label{eq:stage3}
  \hat\betab_w = \arg \min_{\betab \in \RR^{p}}
  \sum_{i \in [n]} \frac{(y_i - \xb_i'\betab)^2}{\hat\sigma_i} + 
  2 n \sum_{j \in [p]} \pen_{\lambda_S}(|\beta_j|)
\end{equation}
where $\hat\sigma_i = \exp(\xb_i'\hat\thetab/2)$ are the weights.  

In classical literature, estimation under heteroscedastic models is
achieved by employing a pseudolikelihood objective.  The
pseudolikelihood maximization principle prescribes the scientist to
maximize a surrogate likelihood, i.e.~one that is believed to be
similar to the likelihood with the true unknown fixed variances (or
means alternatively).  In classical theory, central limit theorems are
derived for many pseudo-maximum likelihood (PML) estimators using
generalized estimating equations \cite{zieglerGEE}.  HIPPO fits neatly
into the pseudolikelihood framework because the first step is a
regularized PML where only the mean structure needs to be correctly
specified.  The second step and third steps may be similarly cast as
PML estimators.  Indeed, all our theoretical results are due to the
fact that in each step we are optimizing a pseudolikelihood that is
similar to the true unknown likelihoods (with alternating free
parameters).  Moreover, it is known that if the surrogate variances in
the mean PML are more similar to the true variances then the resulting
estimates will be more asymptotically efficient.  With this in mind,
we recommend a third reweighting procedure with the variance estimates
from the second step.

\citet{fan01variable} advocate usage of penalty functions that result
in estimates satisfying three properties: unbiasedness, sparsity and
continuity. A reasonable estimator should correctly identify the
support of the true parameter with probability converging to
one. Furthermore, on this support, the estimated coefficients should
have the same asymptotic distribution as if an estimator that knew the
true support was used. Such an estimator satisfies the {\it oracle
  property}. A number of concave penalties result in estimates that
satisfy this property: the SCAD penalty \citep{fan01variable}, the MCP
penalty \citep{zhang2010nearly} and a class of folded concave
penalties \cite{lv09unified}. For concreteness, we choose to use the
SCAD penalty, which is defined by its derivative 
\begin{equation}
\pen_{\lambda}'(t) = \lambda \left[ 
     I\{ t \le \lambda \} +
     \frac{(a \lambda - t)_+}{(a-1)\lambda} I\{t > \lambda\} 
   \right],
\end{equation}
where often $a = 3.7$ is used. Note that estimates produced by the
$\ell_1$-norm penalty are biased, and hence this penalty does not
achieve oracle property.

HIPPO is related to the iterative HHR algorithm of
\citet{daye11high}. In particular, the first two iterations of HHR are
equivalent to HIPPO with the SCAD penalty replaced with the $\ell_1$
norm penalty. In practice, one can continue iterating between solving
\eqref{eq:stage2} and \eqref{eq:stage3}, however, establishing
theoretical properties for those iterates is a non-trivial task. From
our numerical studies, we observe that HIPPO performs well when
stopped after the first two iterations.

\subsection{Tuning Parameter Selection}

As described in the previous section, HIPPO requires selection of the
tuning parameters $\lambda_S$ and $\lambda_T$, which balance the
complexity of the estimated model and the fit to data. A common
approach is to form a grid of candidate values for the tuning
parameters $\lambda_S$ and $\lambda_T$ and chose those that minimize
the AIC or BIC criterion
\begin{equation}
  \label{eq:AIC-criterion}
  {\rm AIC}(\lambda_S, \lambda_T) = 
   \sum_{i \in [n]} \ell(y_i, \xb_i; \hat\betab, \hat\thetab)
   + 2 \widehat{df},
\end{equation}
\begin{equation}
  \label{eq:BIC-criterion}
  {\rm BIC}(\lambda_S, \lambda_T) = 
   \sum_{i \in [n]} \ell(y_i, \xb_i; \hat\betab, \hat\thetab)
   +  \widehat{df} \log n
\end{equation}
where, up to constants,
\[
\ell(y, \xb; \betab, \thetab) = 
\xb'\thetab + (y - \xb'\betab)^2\exp(-\xb'\thetab)
\]
is the negative log-likelihood and 
\[
\widehat{df} = |{\rm supp}(\hat\betab)| + |{\rm supp}(\hat\thetab)|
\]
is the estimated degrees of freedom. In Section~\ref{sec:simulation},
we compare performance of the AIC and the BIC for HIPPO in a
simulation study.

\subsection{Optimization Procedure}

In this section, we describe numerical procedures used to solve
optimization problems in \eqref{eq:stage1}, \eqref{eq:stage2} and
\eqref{eq:stage3}. Our procedures are based on the local linear
approximation for the SCAD penalty developed in \cite{zou08onestep},
which gives:
\[
\begin{aligned}
\pen_{\lambda}(|\beta_j|) \approx
\pen_{\lambda}(|\beta_j^{(k)}|) + 
\pen_{\lambda}'(|\beta_j^{(k)}|)&(|\beta_j| - |\beta_j^{(k)}|),\\ 
&\qquad \text{for }\beta_j \approx \beta_j^{(k)}.
\end{aligned}
\]
This approximation allows us to substitute the SCAD penalty $\sum_{j
  \in [p]}\pen_{\lambda}(|\beta_j|)$ in \eqref{eq:stage1},
\eqref{eq:stage2} and \eqref{eq:stage3} with
\begin{equation}  
  \label{eq:scad_sub}
  \sum_{j \in [p]}\pen_{\lambda}'(|\hat\beta_j^{(k)}|)|\beta_j|, 
\end{equation}
and iteratively solve each objective until convergence of
$\{\hat\betab^{(k)}\}_k$. We set the initial estimates
$\hat\betab^{(0)}$ and $\hat\thetab^{(0)}$ to be the solutions of the
$\ell_1$-norm penalized problems.  The convergence of these iterative
approximations follows from the convergence of the MM
(minorize-maximize) algorithms \citep{zou08onestep}.

With the approximation of the SCAD penalty given in
\eqref{eq:scad_sub}, we can solve \eqref{eq:stage1} and
\eqref{eq:stage3} using standard lasso solvers, e.g., we use the
proximal method of \citet{beck09fast}. The objective in
\eqref{eq:stage2} is minimized using a coordinate descent algorithm,
which is detailed in \citet{daye11high}.

\section{Theoretical Properties of \acro}

In this section, we present theoretical properties of HIPPO. In
particular, we show that HIPPO achieves the oracle property for
estimating the mean and variance under the model \eqref{eq:model}.
All the proofs are deferred to Appendix. 

We will analyze HIPPO under the following assumptions, which are
standard in the literature on high-dimensional statistical learning
\citep[see, e.g.][]{fan12variance}.

{\it Assumption 1.}  The matrix $\Xb = (\xb_1, \ldots, \xb_n)' \in
\RR^{n \times p}$ has independent rows that satisfy $\xb_i =
\Sigmab^{1/2} \zb_i$ where $\{\zb_i\}_i$ are \iid~subgaussian random
variables with $\EE \zb_i = \zero$, $\EE \zb_i\zb_i' = \Ib$ and
parameter $K$ (see Appendix for more details on subgaussian random
variables). Furthermore, there exist two constants $C_{\min}, C_{\max}
> 0$ such that
\[
  0 < C_{\min} \leq \Lambda_{\min}(\Sigmab) \leq
  \Lambda_{\max}(\Sigmab) \leq C_{\max} < \infty.
\]

{\it Assumption 2.} The errors $\epsilon_1, \ldots, \epsilon_n$ are
\iid~subgaussian with zero mean and parameter $1$.

{\it Assumption 3.} There are two constants $\bar\beta$ and
$\bar\theta$ such that $\norm{\betab} \leq \bar \beta < \infty$ and
$\norm{\thetab} \leq \bar\theta < \infty$.

{\it Assumption 4.} $|S| = C_S n^{\alpha_S}$ and $|T| = C_T
n^{\alpha_T}$ for some $\alpha_S \in (0,1)$ and $\alpha_T \in (0,1/3)$ and constants
$C_S, C_T > 0$.

The following assumption will be needed for showing the consistency of
the weighted estimator $\hat\betab_w$ in \eqref{eq:stage3}. 

{\it Assumption 5.}  Define 
\[
\Db_{SS} = n^{-1}\Xb_S'\diag(\exp(-\Xb\theta))\Xb_S.
\]
There exist constants $0 \leq D_{\min}, D_{\max} \leq \infty$ such that
\begin{equation*}
\begin{aligned}
  \lim_{n \rightarrow \infty}&\PP[\Lambda_{\max}(\Db_{SS}) \leq D_{\max}]=1,
  \qquad\text{ and } \\
  \lim_{n \rightarrow \infty}&\PP[\Lambda_{\min}(\Db_{SS}) \geq D_{\min}]=1.
\end{aligned}
\end{equation*}
Furthermore, we have that
\begin{equation*}
\lim_{n \rightarrow \infty} \opnorm{\Db_{SS} - \EE\Db_{SS}}{2} = o_P(1).
\end{equation*}

With these assumption, we state our first result, regarding the
estimator $\hat\betab$ in \eqref{eq:stage1}. 

\begin{theorem}
  \label{thm:ols_mean}
  Suppose that the assumptions (1)-(4) are satisfied. Furthermore,
  assume that $\lambda_S \geq
  c_1\sqrt{\log(p)\exp(\sqrt{c_2\log(n)})/n}$, $\min_{j \in [S]}
  |\beta_j| \gg \lambda_S \gg
  c_3\sqrt{\log(s)\exp(\sqrt{c_2\log(n)})/n}$ and $\log(p) =
  \Ocal(n^{\alpha_0})$ for some $\alpha_0 \in (0,1)$. Then there is a
  strict local minimizer $\hat\betab = (\hat\betab_S', \zero_{S^C}')'$
  of \eqref{eq:stage1} that satisfies
  \begin{equation}
    \label{eq:bound_stage_1}
    \norm{\hat\betab_S - \betab_S}_\infty \leq c_3\sqrt{\frac{\exp(\sqrt{c_2\log(n)})\log(s)}{n}}
  \end{equation}
  for some positive constants $c_1, c_2$, and $c_3$ and
  sufficiently large $n$.

  In addition, if we suppose that assumption (5) is satisfied, then
  for any fixed $\ab \in \RR^{s}$ with $\norm{\ab}_2 = 1$ the
  following weak convergence holds
  \begin{equation}
    \label{eq:variance_stage_1}
    \frac{\sqrt{n}}{\zeta}\ab'(\hat\betab_S - \betab_S)
    \stackrel{D}{\longrightarrow} \Ncal(0, 1)
  \end{equation}
  where $\zeta^2 =
  \ab'\Sigmab_{SS}^{-1}\EE\Db_{SS}\Sigmab_{SS}^{-1}\ab$.
\end{theorem}
The first result stated in Theorem~\ref{thm:ols_mean} established that
$\hat\betab$ achieves the weak oracle property in the sense of
\cite{lv09unified}. The extra term $\exp(\sqrt{\log n})$ is
subpolynomial in $n$ and appears in the bound \eqref{eq:bound_stage_1}
due to the heteroscedastic nature of the errors.  The second result
establishes the strong oracle property of the estimator $\hat \betab$
in the sense of \cite{fan09nonconcave}, that is, we establish the
asymptotic normality on the true support $S$. The asymptotic normality
shows that $\hat\betab_S$ has the same asymptotic variance as the
ordinary least squares (OLS) estimator on the true support. However,
in the case of a heteroscedastic model the OLS estimator is dominated
by the generalized least squares estimator. Later in this section, we
will demonstrate that $\hat\betab_w$ has better asymptotic
variance. Note that $\hat\betab$ correctly selects the mean model and
estimates the parameters at the correct rate. From the upper and lower
bounds on $\lambda_S$, we see how the rate at which $p$ can grow and
the minimum coefficient size are related. Larger the ambient dimension
$p$ gets, larger the size of $\lambda_S$, which lower bounds the size
of the minimum coefficient.


Our next result establishes correct model selection for the variance
parameter $\thetab$.
\begin{theorem}
  \label{thm:var_est}
  Suppose that assumptions (1)-(5) are satisfied.
  Suppose further that $\lambda_T \ge n^{\alpha_T - 1/2}\log(p)\log(n)$ and $\min_{j \in [T]} |\theta_j| \ge \lambda_T$.
  Then there is a strict local minimizer $\hat\thetab = (\hat\thetab'_T,\zero'_{T^C})'$ with the strong oracle property,
\[
  \norm{n^{(1 - \alpha_T)/2} (\hat{\thetab} - \thetab) } = O_\PP (1)
\]
  Morover, for any fixed $\ab \in \RR^{t}$ with $\norm{\ab}_2 = 1$ the
  following weak convergence holds
  \begin{equation}
    \label{eq:variance_stage_2}
    \frac{\sqrt{n}}{\zeta}\ab'(\hat\thetab_T - \thetab_T)
    \stackrel{D}{\longrightarrow} \Ncal(0, 1)
  \end{equation}
  where $\zeta^2 =
  \ab'\Sigmab_{TT}^{-1}\ab$.
\end{theorem}


With the convergence result of $\hat\thetab$ we can prove consistency
and asymptotic normality of the weighted estimator $\hat\betab$ in
\eqref{eq:stage3}. 
\begin{theorem}
  \label{thm:wls_mean}

  Suppose that the assumptions (1)-(5) are satisfied and that there
  exists an estimator $\hat\thetab$ satisfying $\norm{\hat\thetab -
    \thetab}_2 = \Ocal(r_n)$, for a sequence $r_n \rightarrow 0$ and
  $\supp(\hat\thetab) = \supp(\thetab)$. Furthermore, assume that
  $\lambda_S \geq c_1\sqrt{\log(p)\exp(\sqrt{c_2\log(n)})/n}$,
  $\min_{j \in [S]} |\beta_j| \gg \lambda_S \gg
  c_3r_n\exp(\sqrt{c_2\log(n)})\log(n)$ and $\log(p) =
  \Ocal(n^{\alpha_0})$ for some $\alpha_0 \in (0,1)$. Then there is a
  strict local minimizer $\hat\betab_w = (\hat\betab_{w,S}',
  \zero_{S^C})$ of \eqref{eq:stage3} that satisfies
  \begin{equation}
    \label{eq:bound_stage_3}
    \norm{\hat\betab_{w,S} - \betab_S}_\infty \leq c_3 r_n
    \exp(\sqrt{c_2\log(n)}) \log(n)
  \end{equation}
  for some positive constants $c_1, c_2$, and $c_3$ and sufficiently
  large $n$.

  Furthermore, for any fixed $\ab \in \RR^{s}$ with $\norm{\ab}_2 = 1$
  the following weak convergence holds
  \begin{equation}
    \label{eq:variance_stage_3}
    \frac{\sqrt{n}}{\zeta_w}\ab'(\hat\betab - \beta)
    \stackrel{D}{\longrightarrow} \Ncal(0, 1)
  \end{equation}
  where $\zeta_w^2 =
  \ab'(\EE\Db_{SS})^{-1}\ab$.
\end{theorem}

Theorem~\ref{thm:wls_mean} establishes convergence of the weighted
estimator $\hat\betab_W$ in \eqref{eq:stage3} and the model selection
consistency. The rate of convergence depends on the rate of
convergence of the variance estimator, $r_n$. From
Theorem~\ref{thm:var_est}, we show the parametric rate of convergence
for $\hat\theta_S$. The second result of Theorem~\ref{thm:wls_mean}
states that the weighted estimator $\hat\betab_{w,S}$ is
asymptotically normal, with the same asymptotic variance as the
generalized least squares estimator which knows the true model and
variance function $\sigma(\xb, \thetab)$.


\section{Monte-Carlo Simulations}
\label{sec:simulation}

\begin{table}[t]
{
\hfill{}
\begin{tabular}{l@{\hspace{0.3cm}}ccc}
  &  $\norm{\theta - \hat\theta}_2$  &
     ${\sf Pre}_{\theta}$ & ${\sf Rec}_{\theta}$ \\
\cline{2-4}
\vspace{-0.2cm}
\\
&\multicolumn{3}{c}{ \underline{$\rho = 0$} } \\
\vspace{-0.2cm}
\\
HHR-AIC  & 0.59(0.13) & 0.4(0.17) & 1.00(0.00) \\
HIPPO-AIC & 0.26(0.15) & 0.6(0.22) & 1.00(0.00) \\
HHR-BIC & 0.59(0.13) & 0.39(0.16) & 1.00(0.00) \\
HIPPO-BIC & 0.26(0.15) & 0.59(0.22) & 1.00(0.00) \\
\\
&\multicolumn{3}{c}{ \underline{$\rho = 0.5$} } \\
\vspace{-0.2cm}
\\
HHR-AIC  & 0.32(0.12) & 0.68(0.21) & 1.00(0.00) \\
HIPPO-AIC & 0.38(0.22) & 0.69(0.25) & 1.00(0.03) \\ 
HHR-BIC & 0.32(0.12) & 0.68(0.21) & 1.00(0.00) \\
HIPPO-BIC & 0.38(0.22) & 0.69(0.25) & 0.99(0.03) \\
\hline
\hline
\end{tabular}
}
\hfill{}
\caption{
  Mean (sd) performance of HHR and \acro under the model in
  Example~1 (averaged over 100 independent runs). The mean parameter
  $\betab$ is assumed to be known.
}
\label{tb:zero_mean}
\end{table}

In this section, we conduct two small scale simulation studies to
demonstrate finite sample performance of \acro. We compare it to the
HHR procedure \citep{daye11high}.

Convergence of the parameters is measured in the $\ell_2$ norm,
$\norm{\hat \beta - \beta}$ and $\norm{\hat \theta - \theta}$.  We
measure the identification of the support of $\betab$ and $\thetab$
using precision and recall. Let $\hat S$ denote the estimated set of
non-zero coefficients of $S$, then the precision is calculated as
${\sf Pre}_{\beta} := |\hat S \cap S|/|\hat S|$ and the recall as
${\sf Rec}_{\beta} := |\hat S \cap S|/|S|$. Similarly, we can define
precision and recall for the variance coefficients. We report results
averaged over 100 independent runs.

\begin{table*}[t]
{
\hfill{}
\begin{tabular}{ll@{\hspace{1cm}}ccc@{\hspace{1cm}}ccc}
& \#it & $\norm{\beta - \hat\beta}_2$  &
     ${\sf Pre}_{\beta}$ & ${\sf Rec}_{\beta}$ 
   & $\norm{\theta - \hat\theta}_2$  &
     ${\sf Pre}_{\theta}$ & ${\sf Rec}_{\theta}$ \\
\cline{3-8}
\vspace{-0.2cm}
\\
&&\multicolumn{6}{c}{\underline{$n = 200$}} \\
\vspace{-0.2cm}
\\

HHR-AIC  & 1st & 0.78(0.52) & 0.44(0.22) & 1.00(0.00) 
           & 2.10(0.11) & 0.25(0.10) & 0.54(0.16) \\
     & 2nd & 0.31(0.13) & 0.88(0.15) & 1.00(0.00) 
           & 1.80(0.16) & 0.29(0.07) & 0.71(0.14) \\
HIPPO-AIC  & 1st & 0.66(0.84) & 0.75(0.29) & 1.00(0.02) 
           & 2.00(0.16) & 0.20(0.10) & 0.52(0.16) \\
     & 2nd & 0.08(0.07) & 0.84(0.24) & 1.00(0.00) 
           & 1.50(0.30) & 0.30(0.11) & 0.75(0.12) \\
\\
HHR-BIC  & 1st & 0.77(0.48) & 0.58(0.17) & 1.00(0.00) 
           & 2.10(0.10) & 0.41(0.18) & 0.45(0.14) \\
     & 2nd & 0.31(0.13) & 0.89(0.13) & 1.00(0.00) 
           & 1.90(0.16) & 0.38(0.15) & 0.65(0.17) \\
HIPPO-BIC  & 1st & 0.70(0.83) & 0.80(0.25) & 0.99(0.03) 
           & 2.00(0.14) & 0.39(0.18) & 0.50(0.17) \\
     & 2nd & 0.08(0.06) & 0.97(0.07) & 1.00(0.00) 
           & 1.60(0.28) & 0.44(0.16) & 0.72(0.14) \\
\\
&&\multicolumn{6}{c}{\underline{$n = 400$}} \\
\vspace{-0.2cm}
\\
HHR-AIC  & 1st & 0.59(0.37) & 0.58(0.26) & 1.00(0.00) 
               & 1.90(0.11) & 0.36(0.14) & 0.72(0.18) \\
     & 2nd & 0.30(0.24) & 0.98(0.06) & 1.00(0.00) 
           & 1.70(0.16) & 0.43(0.13) & 0.81(0.16) \\
HIPPO-AIC  & 1st & 0.44(0.54) & 0.87(0.22) & 1.00(0.00) 
                 & 1.80(0.18) & 0.28(0.10) & 0.67(0.15) \\
     & 2nd & 0.06(0.29) & 0.97(0.12) & 1.00(0.02) 
           & 1.00(0.31) & 0.56(0.18) & 0.93(0.09) \\
\\
HHR-BIC  & 1st & 0.59(0.37) & 0.66(0.20) & 1.00(0.00) 
           & 1.90(0.11) & 0.46(0.18) & 0.66(0.20) \\
     & 2nd & 0.30(0.23) & 0.98(0.06) & 1.00(0.00) 
           & 1.70(0.17) & 0.46(0.13) & 0.80(0.17) \\
HIPPO-BIC  & 1st & 0.46(0.58) & 0.89(0.19) & 1.00(0.01) 
           & 1.80(0.18) & 0.39(0.17) & 0.65(0.17) \\
     & 2nd & 0.06(0.29) & 0.99(0.06) & 1.00(0.02) 
           & 1.00(0.31) & 0.63(0.20) & 0.92(0.09) \\
\hline
\hline
\end{tabular}
}
\hfill{}
\caption{
  Mean (sd) performance of HHR and \acro under the model in
  Example~2 (averaged over 100 independent runs). We report estimated
  models after the first and second iteration. 
}
\label{tb:exper2}
\end{table*}

\subsection{Example 1} 

Assume that the data is generated iid from the following model $Y =
\sigma(\Xb) \epsilon$ where $\epsilon$ follows a standard normal
distribution and the logarithm of the variance is given by
\[
\log \sigma(\Xb)^2 = X_1 + X_2 + X_3.
\]
The covariates associated with the variance are jointly normal with
equal correlation $\rho$, and marginally $\Ncal(0,1)$. The remaining
covariates, $X_4, \ldots, X_p$ are iid random variables following the
standard Normal distribution and are independent from $(X_1, X_2,
X_3)$. We set $(n, p) = (200, 2000)$ and use $\rho = 0$ and $\rho =
0.5$. Estimation procedures know that $\betab = \zero$ and we only
estimate the variance parameter $\thetab$. This example is provided to
illustrate performance of the penalized pseudolikelihood estimators in
an idealized situation. When the mean parameter needs to be estimated
as well, we expect the performance of the procedures only to get
worse. Since the mean is known, both HHR and \acro only solve the
optimization procedure in \eqref{eq:stage2}, HHR with the
$\ell_1$-norm penalty and \acro with the SCAD penalty, without
iterating between \eqref{eq:stage3} and \eqref{eq:stage2}.

Table~\ref{tb:zero_mean} summarizes the results. Under this toy model,
we observe that HIPPO performs better than HHR when the correlation
between the relevant predictors is $\rho = 0$. However, we do not
observe the difference between the two procedures when $\rho =
0.5$. The difference between the AIC and BIC is already visible in
this example when $\rho = 0$. The AIC tends to pick more complex
models, while the BIC is more conservative and selects a model with
fewer variables. 

\subsection{Example 2} 

The following non-trivial model is borrowed from \citet{daye11high}.
The response variable $Y$ satisfies 
\[
Y = \beta_0 + \sum_{j \in [p]} X_j \beta_j + 
\exp(\theta_0 + \sum_{j \in [p]} X_j\theta_j) \epsilon 
\]
with $p=600$, $\beta_0 = 2$, $\theta_0 = 1$,
\[
\betab_{[12]} = (3, 3, 3, 1.5, 1.5, 1.5, 0, 0, 0, 2, 2, 2)',
\]
\[
\thetab_{[15]} = (1, 1, 1, 0, 0, 0, 0.5, 0.5, 0.5, 0, 0, 0, 0.75, 0.75, 0.75)',
\]
and the remainder of the coefficients are $0$. The covariates are
jointly Normal with ${\rm cov}(X_i, X_j) = 0.5^{|i-j|}$ and the error
$\epsilon$ follows the standard Normal distribution.  This is a more
realistic model than the one described in the previous example. We set
$p = 600$ and the number of samples $n = 200$ and $n = 400$.

Table~\ref{tb:exper2} summarizes results of the simulation. We observe
that HIPPO consistently outperforms HHR in all scenarios. Again, a
general observation is that the AIC selects more complex models
although the difference is less pronounced when the sample size
$n=400$. Furthermore, we note that the estimation error significantly
reduces after the first iteration, which demonstrates final sample
benefits of estimating the variance. Recall that
Theorem~\ref{thm:ols_mean} proves that the estimate $\hat\betab$
consistently estimates the true parameter $\betab$. However, it is
important to estimate the variance parameter $\thetab$ well, both in
theory (see Theorem~\ref{thm:wls_mean}) and practice.

\section{Real Data Application}

Forecasting the gross domestic product (GDP) of a country based on
macroeconomic indicators is of significant interest to the economic
community.  We obtain both the country GDP figures (specifically we
use the GDP per capita using current prices in units of a `national
currency') and macroeconomic variables from the International Monetary
Fund's World Economic Outlook (WEO) database.  The WEO database
contains records for macroeconomic variables from 1980 to 2016 (with
forecasts).

To form our response variable, $\yb_{i,t}$, we form log-returns of the
GDP for each country ($i$) for each time point ($t$) after records
began and before the forecasting commenced (each country had a
different year at which forecasting began).  After removing missing
values, we obtained 31 variables that can be grouped into a few broad
categories: balance of payments, government finance and debt,
inflation, and demographics.  We apply various transformations,
including lagging and logarithms forming the vectors $\xb_{i,t}$.  We
fit the heteroscedastic AR(1) model with HIPPO.
\[
\yb_{i,t} = \xb_{i,t-1}' \betab + 
     \exp (\xb_{i,t-1}' \thetab) \epsilon_{i,t}
\]

In order to initially assess the heteroscedasticity of the data, we
form the LASSO estimator with the LARS package in R selecting with
BIC.  It is common practice when diagnosing heteroscedasticity to plot
the studentized residuals against the fitted values.  We bin the bulk
of the samples into three groups by fitted values, and observe the
box-plot of each bin by residuals (Figure \ref{fig:studres}).  It is
apparent that there is a difference of variances between these bins,
which is corroborated by performing a F-test of equal variances across
the second and third bins (p-value of $4\times10^{-6}$).  We further
observe differences of variance between country GDP log returns.  We
analyzed the distribution of responses separated by countries: Canada,
Finland, Greece and Italy.  The p-value from the F-test for equality
of variances between the countries Canada and Greece is $0.008$, which
is below even the pairwise Bonferroni correction of $0.0083$ at $0.05$
significance level.
This demonstrates heteroscedasticity in the WEO dataset, and we are justified in fitting non-constant variance.

We compare the results from HIPPO and HHR when applied to the WEO
data set.  The tuning parameters were selected with BIC over a grid for
$\lambda_S$ and $\lambda_T$.  The metrics used to compare the
algorithms are mean square error (MSE) defined by $\frac{1}{n} \sum_i
(y_{i,t} - \hat y_{i,t})^2$, the partial prediction score defined as
the average of the negative log likelihoods, and the number of
selected mean parameters and variance parameters.  We perform
$10$-fold cross validation to obtain unbiased estimates of these
metrics.  In Table \ref{tb:imf} we observe that HIPPO outperforms HHR
in terms of MSE and partial prediction score.

\begin{table}[t]
{
\hfill{}
\begin{tabular}{c@{\hspace{0.6cm}}cc}
& \underline{\acro} & \underline{HHR} \\ 
\vspace{-0.2cm}
\\
MSE         & 0.0089 & 0.0091 \\
$-\ell(\yb, \Xb; \hat\betab, \hat\thetab)$    & 0.4953 & 0.6783 \\
 $|\widehat S|$ & 5.4    & 8.9    \\
 $|\widehat T|$ & 8.2    & 5.1    \\
\hline
\hline
\end{tabular}
}
\hfill{}
\caption{
  Performance of \acro and HHR on WEO data averaged over 10 folds.
}
\label{tb:imf}
\end{table}


\section{Discussion}

We have addressed the problem of statistical inference in
high-dimensional linear regression models with heteroscedastic errors.
Heteroscedastic errors arise in many applications and industrial
settings, including biostatistics, finance and quality control in
manufacturing. We have proposed HIPPO for model selection and
estimation of both the mean and variance parameters under a
heteroscedastic model. HIPPO can be deployed naturally into an
existing data analysis work-flow. Specifically, as a first step, a
statistician performs penalized estimation of the mean parameters and
then, as a second step, tests for heteroscedasticity by running the
second step of HIPPO. If heteroscedasticity is discovered, HIPPO can
then be used to solve penalized generalized least squares
objective. Furthermore, HIPPO is well motivated from the penalized
pseudolikelihood maximization perspective and achieves the oracle
property in high-dimensional problems. 

Throughout the paper, we focus on a specific parametric form of the
variance function for simplicity of presentation. Our method can be
extended to any parametric form, however, the assumptions will become
more cumbersome and the particular numerical procedure would
change. It is of interest to develop general unified framework for
estimation of arbitrary parametric form of the variance function. 
Another open research direction includes non-parametric estimation of
the variance function in high-dimensions, which could be achieved with
sparse additive models \citep[see ][]{ravikumar2009sparse}.



\begin{figure}[p]
\centering
\includegraphics[width=7cm]{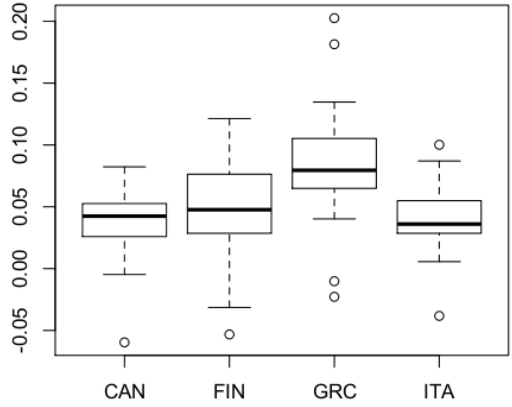}
\caption{A box-plot of the GDP log-returns for the 4 countries with
  the most observed time points (Canada, Finland, Greece, and Italy).}
\label{fig:country}
\end{figure}

\begin{figure}[p]
\centering
\includegraphics[width=7cm]{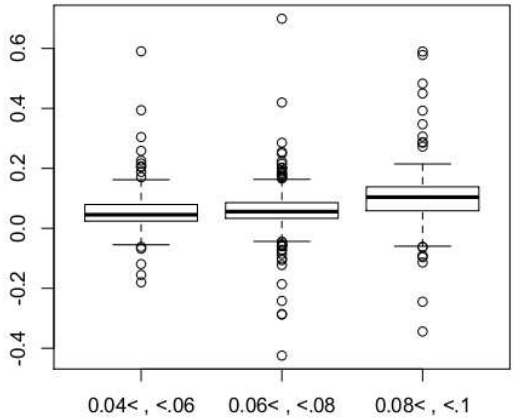}
\caption{A box-plot of the studentized residuals binned by LASSO
  predicted $y_{i,t}$.  Only the segment of the predicted response
  with the bulk of the samples was binned; the breaks in the bins are
  at $0.04$, $0.06$, $0.08$, and $0.1$.}
\label{fig:studres}
\end{figure}

\section*{Acknowledgements} MK is partially supported through the grants NIH R01GM087694 and AFOSR FA9550010247.
JS is partially supported by AFOSR under grant FA95501010382.


\bibliographystyle{plainnat}
\bibliography{biblio}

\section{Appendix}

In the appendix, we collect some well known results and provide proofs
for the results in the main text.

For readers convenience, we summarize the notation again. We use $[n]$
to denote the set $\{1,\ldots,n\}$.  For any index set $S \subseteq
[p]$, we denote $\betab_S$ to be the subvector containing the
components of the vector $\betab$ indexed by the set $S$, and $\Xb_S$
denotes the submatrix containing the columns of $\Xb$ indexed by
$S$. For a vector $\ab \in \RR^n$, we denote ${\rm supp}(\ab) = \{j\
:\ a_j \neq 0\}$ the support set, $\norm{\ab}_q$, $q \in (0,\infty)$,
the $\ell_q$-norm defined as $\norm{\ab}_q = (\sum_{i\in[n]}
a_i^q)^{1/q}$ with the usual extensions for $q \in \{0,\infty\}$, that
is, $\norm{\ab}_0 = |{\rm supp}(\ab)|$ and $\norm{\ab}_\infty =
\max_{i\in[n]}|a_i|$. For notational simplicity, we denote
$\norm{\cdot} = \norm{\cdot}_2$ the $\ell_2$ norm.  The unit sphere of
$\ell_2^n = (\RR^n, \norm{\cdot})$ is denoted by $\Scal^{n-1}$. The
canonical bases of $\ell_2^n$ we denote by $\eb_1, \ldots, \eb_n$. For
a matrix $\Ab \in \RR^{n \times p}$ we denote $\opnorm{\Ab}{2} =
\sup\{\norm{\Ab x}:\norm{x} = 1\}$ the operator norm and $\norm{\Ab}_F
= \sqrt{\sum_{i \in [n]}\sum_{j \in [p]} a_{ij}^2}$ the Frobenius
norm. For a symmetric matrix $\Ab \in \RR^{n \times n}$, we use
$\Lambda_{\min}(\Ab)$ to denote $\Lambda_{\max}(\Ab)$ the smallest and
largest eigenvalue respectively.

\subsection{Subgaussian random variables}
\label{sec:subgaussian}

In this section, we define subgaussian random variables and state a
few well known properties.

We denote $\norm{X}_{L_p}$ the $L_p$ norm of a random variable $X$,
i.e., $\norm{X}_{L_p} = (EE|X|^p)^{1/p}$. Define $\Psi_\alpha$ the
Orlitz function $\Psi_\alpha(x) = \exp(|x|^\alpha) - 1$, $\alpha \geq
1$. Using the Orlitz function, we can define the Orlitz space of real
valued random variables, $L_{\Psi_2}$, equipped with the norm
\[
  \norm{X}_{\Psi_\alpha} = \inf \{c>0 : \EE\exp(|X/c|^\alpha)\leq 2 \}.
\]
We will focus on the particular choice of $\alpha = 2$. Define
$B_{\Psi_2}(\gamma)$ the set of real-valued symmetric random variables
satisfying
\begin{equation}
  \label{eq:def_b_gamma}
  1 \leq \norm{X}_{L_2} \text{ and } \norm{x}_{\Psi_2} \leq \gamma.
\end{equation}
For $X \in B_{\Psi_2}(\gamma)$ we have  a good control of the tail
probability 
\begin{equation}
  \label{eq:tail_bound_sub_gauss}
  \PP[X \geq t] \leq \exp(- t^2 / \gamma^2),
\end{equation}
which can be obtained using the Markov inequality
\[
2\PP[X \geq t] 
\leq \PP[|X| \geq t] 
\leq \frac{\EE\exp(X^2/\gamma^2)}{\exp(t^2/\gamma^2)}
\leq 2\exp(-t^2/\gamma^2)
\]
since $X$ is symmetric and $\EE\exp(X^2/\gamma^2) \leq 2$.

The space $L_{\Psi_2}$ is the set of subgaussian random
variables. A real-valued random variable $X$ is called subgaussian
with parameter $\nu$, $\nu >0$, if
\begin{equation}
  \label{eq:nu-subgaussian}
  \EE\exp(tX) \leq \exp(\nu^2t^2/2), \text{ for all } t > 0.
\end{equation}
It follows from this bound on the moment generating function that the
following bound on the tail probability holds
\begin{equation}
  \label{eq:tail_bound_nu_sub_gaussian}
  \PP[X \geq t] \leq \exp(-t^2 / (2\nu^2)) \text{ for any } t \geq 0.
\end{equation}
We also have that $X \in B_{\Psi_2}(\mu)$ is subgaussian with
parameter $\sqrt{2}\mu$ by direct calculation.

The following few facts are useful.
\begin{lemma}
  \label{lem:sum_sub_gaussian}
  Let $\gamma_i \geq 1$ and $X_i \in B_{\Psi_2}(\gamma_i)$, $i=1,\ldots,n$,
  be independent variables, then for any $a_1, \ldots, a_n \in \RR$,
  $\sum_{i \in [n]} a_iX_i$ is subgaussian with parameter
  $\sqrt{2\sum_{i\in[n]}\gamma_i^2a_i^2}$. 
\end{lemma}
\begin{proof}
  For any $t > 0$, we have
  \[
  \EE\exp(t\sum_{i\in[n]}a_iX_i) 
  = \prod_{i\in[n]}\EE\exp(ta_iX_i)
  \leq \prod_{i\in[n]}\exp(t^2a_i^2\gamma_i^2)
  = \exp(t^2\sum_{i\in[n]}a_i^2\gamma_i^2).
  \]
  The claim follows from \eqref{eq:nu-subgaussian}.
\end{proof}
\begin{lemma}
  \label{lem:norm_sub_gaussian}
  Let $\gamma \geq 1$ and $X_i \in B_{\Psi_2}(\gamma)$, $i=1,\ldots,n$,
  be independent variables, then for any $u \geq 0$,
  \begin{equation}
    \label{eq:norm_sub_gaussian}
    \PP[\sum_{i\in[n]} X_i^2 \geq u^2n] \leq \exp(n(\log(2) - (u/\gamma)^2)).
  \end{equation}
\end{lemma}
\begin{proof}
  Using Markov inequality we have,
  \[
  \PP[\sum_{i\in[n]} X_i^2 \geq u^2n] 
  \leq \exp(-\gamma^{-2}u^2n)\EE\exp(\gamma^{-2}\sum_{i\in[n]}X_i^2)
  \leq 2^n\exp(-\gamma^{-2}u^2n),
  \]
  which concludes the proof.
\end{proof}

The notion of subgaussian random variable can be easily extended to
vector random variables. Let $\Zb \in \RR^p$ be random variable
satisfying $\EE \Zb = \zero$, $\EE \Zb\Zb' = \Ib_p$. The random
variable $\Zb$ is subgaussian with parameter $\nu$ if it satisfies
\begin{equation}
  \label{eq:mv_subgaussian}
  \sup_{\wb \in \Scal^{p-1}} \orlnorm{\dotp{\zb_i}{\wb}}{2} \leq \nu.
\end{equation}
Let $\Xb = \Sigmab^{1/2}\Zb$ where $\Sigmab \in \RR^{p \times p}$
positive definite matrix and $\Sigmab^{1/2}$ the symmetric matrix
square root. The following result is standard in multivariate
statistics \citep{anderson2003}.
\begin{lemma}
\label{lem:regression}
Let $S \subset [p]$ and $j \in [p]$, $j \not\in S$. Then 
\begin{equation}
  \label{eq:mv_regresion}
  X_j = \dotp{\Xb_S}{(\Sigmab_{SS})^{-1}\Sigmab_{Sj}} + E_j
\end{equation}
and $E_j$ is uncorrelated with $\Xb_S$.  
\end{lemma}
\label{lem:var_regression}
The following lemma shows that $E_j$ is subgaussian. 
\begin{lemma}
\label{lem:Ej_subgauss}
  The random variable $E_j$ defined in \eqref{eq:mv_regresion} is
  subgaussian with parameter $K\sqrt{2\Sigma_{j|S}}$, where
  $\Sigma_{j|S} = \Sigma_{jj} - \Sigmab_{jS}(\Sigmab_{SS})^{-1}\Sigmab_{Sj}$.
\end{lemma}
\begin{proof}
  From the definition of $\Xb$ we have that $X_j = \Sigmab{j\cdot}^{1/2}\Zb$
  and $\Xb_S = \Sigmab_{S\cdot}^{1/2}\Zb$. With this, we have
  \[
  E_j = \Zb'(\Sigmab_{\cdot j}^{1/2} - \Sigmab_{\cdot S}(\Sigmab_{SS})^{-1}\Sigmab_{Sj})
  \]
  and
  \[
  \begin{aligned}
  \orlnorm{E_j}{2} 
  & \leq \norm{\Sigmab_{\cdot j}^{1/2} - \Sigmab_{\cdot S}^{1/2}(\Sigmab_{SS})^{-1}\Sigmab_{Sj}}_2
       \orlnorm{\Zb}{2} \\
  & \leq K \sqrt{\Sigma_{jj} - \Sigmab_{jS}(\Sigmab_{SS})^{-1}\Sigmab_{Sj}} \\
  & = K \sqrt{\Sigma_{j|S}}.
  \end{aligned}
  \]
  This concludes the proof.
\end{proof}

Next, we present a few results on spectral norms of random matrices
obtained as sums of random subgaussian vectors outer products.
\begin{lemma}[\citet{hsu2011dimension}]
  \label{lemma:hsu_2}
  Let $\zb_1, \ldots, \zb_n \in \RR^p$ be i.i.d random subgaussian
  vectors with parameter $\nu$, then for all $\delta \in (0,1)$,
  \begin{equation}
    \label{eq:lambda_max}
    \PP[\opnorm{n^{-1}\sum_{i \in [n]}\zb_i\zb_i' - \Ib}{2} > 
    2\epsilon(n,\delta)] \leq \delta
  \end{equation}
  where
  \[
  \epsilon(n,\delta) = \nu^2\left(
    \sqrt{\frac{8(p\log(9)+\log(2/\delta))}{n}} +
    \frac{p\log(9)+\log(2/\delta)}{n}
  \right).
  \]
\end{lemma}
The above result can easily be extended to variables with arbitrary
covariance matrix.
\begin{lemma}
  \label{lemma:hsu_general_sigma}
  Let $\xb_1, \ldots, \xb_n \RR^p$ be independent random vectors
  satisfying $\xb_i = \Sigmab^{1/2}\zb_i$ with $\zb_1, \ldots, \zb_n$
  being independent subgaussian vectors with parameter $\nu$ and
  $\Sigmab^{1/2}$ is the symmetric matrix square root of $\Sigmab$. If
  $\Lambda_{\max}(\Sigmab) < \infty$ and $\Lambda_{\min}(\Sigmab) >
  0$, then for all $\delta \in (0,1)$
  \begin{equation}
    \label{eq:lambda_max_general}
    \PP[\opnorm{n^{-1}\sum_{i \in [n]}\xb_i\xb_i' - \Sigmab}{2} > 
    2\Lambda_{\max}(\Sigmab)\epsilon(n,\delta)] \leq \delta
  \end{equation}
  and 
  \begin{equation}
    \label{eq:lambda_min_general}
    \PP[\opnorm{(n^{-1}\sum_{i \in [n]}\xb_i\xb_i')^{-1} - (\Sigmab)^{-1}}{2} > 
    2\epsilon(n,\delta) / \Lambda_{\min}(\Sigmab)] \leq \delta.
  \end{equation}
\end{lemma}
\begin{proof}
  We have that 
  \[
  \opnorm{n^{-1}\Xb'\Xb - \Sigmab}{2} 
  = \opnorm{\Sigmab^{1/2}(n^{-1}\Zb'\Zb - \Ib_p)\Sigmab^{1/2}}{2}
  \leq \Lambda_{\max}(\Sigmab)\opnorm{n^{-1}\Zb'\Zb - \Ib_p}{2}
  \]
  and \eqref{eq:lambda_max_general} follows from
  \eqref{eq:lambda_max}.

  Similarly, we can write
  \[
  \begin{aligned}
  \opnorm{(n^{-1}\Xb'\Xb)^{-1} - \Sigmab^{-1}}{2} 
  &= \opnorm{\Sigmab^{-1/2}((n^{-1}\Zb'\Zb)^{-1} - \Ib_p)\Sigmab^{-1/2}}{2}\\
  &\leq \Lambda_{\min}^{-1}(\Sigmab)\opnorm{(n^{-1}\Zb'\Zb)^{-1} - \Ib_p}{2}
  \end{aligned}
  \]
  and \eqref{eq:lambda_min_general} follows from
  \eqref{eq:lambda_max}.
\end{proof}

\subsection{Proofs and Technical Results}

For convenience, we restate technical conditions used in the paper.

{\it Assumption 1.}  The matrix $\Xb = (\xb_1, \ldots, \xb_n)' \in
\RR^{n \times p}$ has independent rows that satisfy $\xb_i =
\Sigmab^{1/2} \zb_i$ where $\{\zb_i\}_i$ are \iid~subgaussian random
variables with $\EE \zb_i = \zero$, $\EE \zb_i\zb_i' = \Ib$ and
$\orlnorm{\zb_i}{2} \leq K$. Furthermore, there exist two constants
$C_{\min}, C_{\max} > 0$ such that
\[
  0 < C_{\min} \leq \Lambda_{\min}(\Sigmab) \leq
  \Lambda_{\max}(\Sigmab) \leq C_{\max} < \infty.
\]

{\it Assumption 2.} The errors $\epsilon_1, \ldots, \epsilon_n$ are
\iid~with $\epsilon_i \in B_{\Psi_2}(1)$.

{\it Assumption 3.} There are two constants $\bar\beta$ and
$\bar\theta$ such that $\norm{\betab} \leq \bar \beta < \infty$ and
$\norm{\thetab} \leq \bar\theta < \infty$.

{\it Assumption 4.} $|S| = C_S n^{\alpha_S}$ and $|T| = C_T
n^{\alpha_T}$ for some $\alpha_S \in (0,1)$ and $\alpha_T \in (0,1/3)$ and constants
$C_S, C_T > 0$.

{\it Assumption 5.}  Define 
\[
\Db_{SS} = n^{-1}\Xb_S'\diag(\exp(-\Xb\theta))\Xb_S.
\]
There exist constants $0 \leq D_{\min}, D_{\max} \leq \infty$ such that
\begin{equation*}
\begin{aligned}
  \lim_{n \rightarrow \infty}&\PP[\Lambda_{\max}(\Db_{SS}) \leq D_{\max}]=1,
  \qquad\text{ and } \\
  \lim_{n \rightarrow \infty}&\PP[\Lambda_{\min}(\Db_{SS}) \geq D_{\min}]=1.
\end{aligned}
\end{equation*}

\subsection{Proof of Theorem~\ref{thm:ols_mean}}

We will split the proof in two parts. In the first part, we show that
the vector $\hat\betab = (\hat\betab_S', \zero_{S^C}')'$, where
$\hat\betab_S = (\Xb_S'\Xb_S)^{-1}\Xb_S'\yb$, is a strict local
minimizer of \eqref{eq:stage1} and $S = \{j : \beta_j \neq 0\}$.  In
the second part, we use results for pseudo-maximum likelihood
estimates to establish asymptotic normality of $\hat \betab_S$.

From Theorem~1 in \cite{fan09nonconcave}, we need to show that $\hat
\betab$ satisfies
\begin{equation}
  \label{eq:proof:scad_mean:cond_1}
  \Xb_S'(\yb - \Xb\hat\betab) - n \sgn(\hat\betab_S)\odot
  \pen'_{\lambda}(\hat\betab_S) = \zero,
\end{equation}
\begin{equation}
  \label{eq:proof:scad_mean:cond_2}
  \norm{\Xb\uSc'(\yb - \Xb\hat\betab)}_\infty < n \pen'_\lambda(0+),
\end{equation}
and 
\begin{equation}
  \label{eq:proof:scad_mean:cond_3}
  \Lambda_{\min}(n^{-1}\Xb_S'\Xb_S) > \max_{j \in S}\ \{-\pen''_\lambda(|\hat\beta_j|)\},
\end{equation}
in order to show that $\hat\betab$ is a strict local minimizer.

Define the events
\[
\Acal_1 = \left\{ \max_{i \in [n]} \exp(|\xb_i'\thetab|) 
   \leq \exp\left(
      \sqrt{K^2\Lambda_{\max}(\Sigmab_{TT})\norm{\thetab}_2^2\log(2n/\delta)}
    \right)
   \right\}
\]
where $T = \{j : \theta_j \neq 0 \}$ and
\[
\Acal_2 = \{ \Lambda_{\max}((n^{-1}\Xb_S'\Xb_S)^{-1}) \leq 3 /
\Lambda_{\min}(\Sigmab_{SS}) \}.
\]
To simplify notation, we define
\begin{equation}
  \label{eq:bar_sigma}
  \bar\sigma^2 = \exp\left(
      \sqrt{K^2\Lambda_{\max}(\Sigmab_{TT})\norm{\thetab}_2^2\log(2n/\delta)}
    \right).
\end{equation}
The following two lemma shows that the events $\Acal_1$ and $\Acal_2$
occur with high probability.
\begin{lemma}
  \label{lem:event_1}
  Under the assumptions of Theorem~\ref{thm:ols_mean}, we have that
  $\PP[\Acal_1] \geq 1 - \delta$.
\end{lemma}
\begin{proof}
We have 
\[
\orlnorm{\xb_i'\thetab}{2} 
\leq \norm{\Sigmab^{1/2}\thetab}_2 \orlnorm{\zb_i}{2}
\leq K\norm{\thetab}_2\Lambda_{\max}^{1/2}(\Sigmab_{TT}).
\]
Lemma follows by setting $t = \sqrt{K^2\Lambda_{\max}(\Sigmab_{TT})\norm{\thetab}_2^2\log(2n/\delta)}$
in \eqref{eq:tail_bound_nu_sub_gaussian} and using the union bound.
\end{proof}
\begin{lemma}
  \label{lem:event_2}
  Suppose that the assumptions of Theorem~\ref{thm:ols_mean} are
  satisfied. Furthermore, assume that $n$ is big enough so that
  \[
    K^2\left(
    \sqrt{\frac{8(C_\alpha n^{\alpha_S}\log(9)+\log(2/\delta))}{n}} +
    \frac{C_\alpha n^{\alpha_S}\log(9)+\log(2/\delta)}{n}
  \right) < 1.
  \]
  Then $\PP[\Acal_2] \geq 1 - \delta$.
\end{lemma}
\begin{proof}
  We have that
  \[
  \opnorm{(n^{-1}\Xb_S'\Xb_S)^{-1}}{2} 
  \leq \opnorm{\Sigmab_{SS}^{-1}}{2} 
       + \opnorm{(n^{-1}\Xb_S'\Xb_S)^{-1} - \Sigmab_{SS}^{-1}}{2}
  \leq 3 / \Lambda_{\min}(\Sigmab_{SS})
  \]
  with probability $1 - \delta$ using \eqref{eq:lambda_min_general}
  and the fact that $n$ is large enough so that $\epsilon(n,\delta) <
  1$.
\end{proof}

Recall that $\hat\betab_S$ is an ordinary least squares estimator
using variables in $S$, so that 
\[
\hat\betab_S - \betab_S = (\Xb_S'\Xb_S)^{-1}\Xb_S'\etab =
(\Xb_S'\Xb_S)^{-1}\Xb_S'\diag(e^{\Xb\thetab/2})\epsilonb.
\]
Define
\[
\Mb :=
(\Xb_S'\Xb_S)^{-1}\Xb_S'\diag(e^{\Xb\thetab})\Xb_S(\Xb_S'\Xb_S)^{-1}.
\]
Conditioned on $\Xb$, using Lemma~\ref{lem:sum_sub_gaussian} we have
that $\eb_j'(\hat\betab_S - \betab_S)$ is subgaussian with parameter
$\sqrt{2m_{jj}}$, $j \in S$. Therefore 
\begin{equation}
  \label{eq:proof:scad_mean:1}
  \PP[\norm{\hat\betab_S - \betab_S}_\infty > t\ \big|\ \Xb]
  \leq 2s\exp\left( -\frac{t^2}{2\max_{j \in S} m_{jj}} \right).
\end{equation}
On the event $\Acal_1 \cap \Acal_2$,
\begin{equation*}
  \max_{j \in S} m_{jj} 
   \leq n^{-1}\bar\sigma^2\Lambda_{\max}((n^{-1}\Xb_S'\Xb_S)^{-1}) 
   \leq \frac{3\bar\sigma^2}{\Lambda_{\min}(\Sigmab_{SS})n}.
\end{equation*}
Setting $t = \sqrt{2(\max_{j \in S} m_{jj}) \log(2s/\delta)}$
in~\eqref{eq:proof:scad_mean:1} and conditioning on the event $\Acal_1 \cap
\Acal_2$ and its complement, we have
\begin{equation}
  \label{eq:proof:scad_mean:inf_norm_beta}
  \norm{\hat\betab_S - \betab_S}_\infty 
  \leq \sqrt{\frac{6 \exp(\sqrt{C\log(2n/\delta)})\log(2s/\delta)}{\Lambda_{\min}(\Sigmab_{SS})n}}
\end{equation}
where $C = K^2\Lambda_{\max}(\Sigmab_{TT})\norm{\thetab}_2^2$ with
probability $1 - 3\delta$. Under the assumptions, we have that
$\norm{\hat\betab_S - \betab_S}_\infty \ll \lambda$.

Using the result obtained above, we have that
\begin{equation*}
\begin{aligned}
  \min_{j \in S} |\hat \beta_j| &
    \geq \min_{j \in S} |\beta_j| - \norm{\hat\betab_S -
      \betab_S}_\infty \\
    & \geq \beta_{\min} - \norm{\hat\betab_S - \betab_S}_\infty \\
    & \geq \beta_{\min}/2 \gg \lambda,
\end{aligned}
\end{equation*}
since $\beta_{\min} \geq n^{-\gamma} \log n$ with $\gamma \in (0,
1/2]$. This gives us that $\pen'(\hat \betab_S) = \zero$ and $\max_{j
  \in S} \{-\pen_\lambda''(|\hat \beta_j|)\} = 0$ showing
\eqref{eq:proof:scad_mean:cond_1} and
\eqref{eq:proof:scad_mean:cond_3}.

Using Lemma~\ref{lem:regression} and Lemma~\ref{lem:var_regression}, we
write $\Xb_j \in \RR^n$ as $\Xb_j = \Xb_S \taub_S +
\Eb_j$, $j \in S^C$, with $\Eb_j$ having elements that are subgaussian
with parameter $K\sqrt{\Sigma_{j|S}}$. Therefore 
\[
n^{-1}\Xb_j'(\yb - \Xb\hat\betab) 
= n^{-1}(\Xb_S \taub_S + \Eb_j)'(\Ib - \Pb_S)\yb 
= n^{-1}\Eb_j'(\Ib - \Pb_S)\diag(\exp(\Xb\thetab/2))\epsilonb.
\]
Denote 
\[
n_j = n^{-2} \norm{\Eb_j}_2^2\max_{i \in n} \exp(\xb_i'\thetab)
\]
and observe that $n_j \geq n^{-2}\norm{\Eb_j'(\Ib - \Pb_S)\diag(\exp(\Xb\thetab/2))}_2^2$.
Condition on $\Xb$, then for any $j \in S^C$,
\begin{equation}
  \label{eq:proof:scad_mean:2}
  \PP[|n^{-1}\Eb_j'(\Ib - \Pb_S)\diag(\exp(\Xb\thetab/2))\epsilonb| > t] \leq 2\exp(-t^2/n_j).
\end{equation}
Using the union bound together with Lemma~\ref{lem:norm_sub_gaussian},
\[
  \max_{j \in S^C} \norm{\Eb_j}_2^2 \leq 3K(\max_{j \in S}\Sigma_{j|S})n
\]
with probability at least $1 - (p-s)\exp(-2n)$. 
Conditioning on the event $\Acal_1$ and its complement
\[
\max_{j \in S^C} n_j \leq 
3K(\max_{j \in S}\Sigma_{j|S})n^{-1}\exp(\sqrt{C\log(2n/\delta)})
\]
where $C = K^2\Lambda_{\max}(\Sigmab_{TT})\norm{\thetab}_2^2$ with
probability $1-\delta-(p-s)\exp(-2n)$. Picking 
\[
t = \sqrt{(\max_{j \in S^C} n_j)\log(2(p-s)/\delta)}
\]
in \eqref{eq:proof:scad_mean:2} and combining with the above, we have
shown that 
\[
  \norm{n^{-1}\Xb\uSc'(\yb - \Xb\hat\betab)}_\infty < \sqrt{(\max_{j \in S^C} n_j)\log(2(p-s)/\delta)}
\]
with probability $1-2\delta-(p-s)\exp(-2n)$. Since
\[
  \sqrt{3K(\max_{j \in S}\Sigma_{j|S})n^{-1}\exp(\sqrt{C\log(2n/\delta)})\log(2(p-s)/\delta)}
  \leq \lambda / 2 < \pen'_\lambda(0+)
\]
we have shown that $\hat \betab$ is a strict local minimizer.
This finishes the proof of the first part.

We are now ready to show asymptotic normality of $\hat\betab$.  Define
$\Wb = \diag(\exp(-\Xb\thetab/2))$. From the proof of the first part
and the assumption (5), we have that
\begin{equation*}
\begin{aligned}
  \hat\betab - \betab 
    & = (\Xb_S'\Xb_S)^{-1}\Xb'\Wb\epsilonb \\
    & = n^{-1/2}(\Sigmab_{SS})^{-1}(\EE \Db_{SS})^{1/2} \epsilonb + o_p(1),
\end{aligned}
\end{equation*}
where the small order term is understood under the $L_2$ norm.
Write
\[
\ab'(\hat\betab - \betab) = \sum_{i \in [n]} c_i \epsilon_i
\]
where $c_i = n^{-1/2}\ab'(\Sigmab_{SS})^{-1}(\EE
\Db_{SS})^{1/2}_{\cdot i}$. It follows that 
\[
\sum_{i \in [n]} \Var(c_i\epsilon_i) = n^{-1}
\ab'(\Sigmab_{SS})^{-1}\EE \Db_{SS}(\Sigmab_{SS})^{-1}\ab,
\]
and
\begin{equation*}
\begin{aligned}
  \sum_{i\in[n]} \EE|c_i\epsilon_i|^3 &= 
n^{-3/2}\sum_{i\in[n]} 
   |\ab'(\Sigmab_{SS})^{-1}(\EE \Db_{SS})^{1/2}_{\cdot i}|^{3}
   \EE|\epsilon_i|^3 \\
& \leq Cn^{-3/2} \norm{\ab'(\Sigmab_{SS})^{-1}}^3
       \sum_{i\in[n]} \norm{(\EE \Db_{SS})^{1/2}_{\cdot i}}^{3} \\
& = o(1).
\end{aligned}
\end{equation*}
This allows us to apply the Lyapunov's theorem to conclude the proof
of the theorem.

\subsection{Proof of Theorem~\ref{thm:var_est}}
  Consider an oracle that performs the second stage of \acro with full
  knowledge of the sparsity set $T$, resulting in the estimator
  $\hat{\thetab}_T$, with $\lambda_T = 0$.  (viz.~ $\hat{\thetab}_T$
  is the pseudo-likelihood maximizer by forming the likelihood with
  the estimated residuals from the OLS.)  Then $\hat\thetab =
  (\hat\thetab'_T,\mathbf{0}'_{T^C})'$ is a strict local minimizer of
  the program \eqref{eq:stage2} for $\lambda_T \asymp n^{-1/2 +
    \alpha_T} \log(p)$.
We derive necessary and sufficient conditions for $\hat\beta$ to
be a strict local minimizer of the program \eqref{eq:stage1} akin
to those in Theorem~1 in \cite{fan09nonconcave}.  We show that the
PML is asymptotically equivalent to the maximum likelihood estimator
through a lemma from \cite{pollard93} and by following arguments
similar to \cite{jobsonFuller}.  As opposed to the classical
asymptotic theory, we implicitely construct finite sample results
because $|T|$ is allowed to grow with $n$.

\begin{lemma}
  \label{lem:Mbdd}
  Let $\Mb = \Xb (\Xb' \Xb)^{-1} \Xb'$ then,
\[
\max_{i \in [n]} \| \Mb_i \| = O_\PP \left( \frac{1}{n^{(3-\alpha_S)/4}} \log (n) \right)
\]
\end{lemma}

\begin{proof}
It is known by \cite{hsu2011dimension} that with probability at least $1 - \delta$,
\[
\opnorm{ \frac{1}{n}\sum_i^{[n]} \xb_{i,S}^\top \xb_{i,S}  - \Sigma_{S,S} }{}= O_\PP \left(  \opnorm{\Sigma_{S,S}}{}\sqrt{ \frac{s + \log(1/\delta)}{n}} \right)
\]
Then we have that 
\[
n \| \Mb_i \| = \| \xb_{i,S} (\frac{\Xb_S^\top \Xb_S}{n})^{-1} \Xb_S^\top \| \le  \| \xb_{i,S} \Sigma_{S,S}^{-1} \Xb_S^\top \| + \| \xb_{i,S} ((\frac{\Xb_S^\top \Xb_S}{n})^{-1} -  \Sigma_{S,S}^{-1}) \Xb_S^\top  \|
\]
Controlling first the rightmost term,
\[
\| \xb_{i,S} ((\frac{\Xb_S^\top \Xb_S}{n})^{-1} -  \Sigma_{S,S}^{-1}) \Xb_S^\top  \| \le \frac{\| \xb_{i,S}\| \opnorm{\Xb_S}{} \opnorm{\frac{\Xb_S^\top \Xb_S}{n} -  \Sigma_{S,S}}{} }{\Lambda_{\min}(\frac{\Xb_S^\top \Xb_S}{n}) \Lambda_{\min} (\Sigma_{S,S})}
\]
by a result we obtain below, $\Lambda_{\min}(\frac{\Xb_S^\top \Xb_S}{n})$ and $\Lambda_{\min} (\Sigma_{S,S})$ are bounded below by a constant with high probability.
By the result above, 
\[
\opnorm{\frac{\Xb_S^\top \Xb_S}{n} -  \Sigma_{S,S}}{}  = O_\PP(\sqrt{ \frac{s \log (1/\delta) }{n}})
\]
furthermore we have that $||\Xb_S||  = O_\PP(\sqrt{s \log(1/\delta)})$.
Hence the second term is $O_\PP ((s \log (1 / \delta))^{3/2}/\sqrt{n})$.

Now consider the first term, $\| \xb_{i,S} \Sigma_{S,S}^{-1} \Xb_S^\top \|$.
And write $U_j = \Sigma_{S,S}^{1/2} \Xb_{j,S}$ then we have that
\[
\| \xb_{i,S} \Sigma_{S,S}^{-1} \Xb_S^\top \|^2 \le \|U_i\|^2 + |\sum_{j \ne i}^{[n]} U_i^\top U_j| \le \|U_i\|^2 + \|U_i\| \sqrt{2 n \log (2/\delta)}
\]
because
\[
\sum_{j \ne i}^{[n]} U_i^\top U_j \le \|U_i\| \sqrt{2 n \log (2/\delta)}
\]
with probability $1 - \delta$ by the sub-Gaussianity of $\{U_j\}$.
Thus, $\| \xb_{i,S} \Sigma_{S,S}^{-1} \Xb_S^\top \| = O ( \sqrt{s \log (1/\delta) + \sqrt{2 sn} \log (2/\delta)} )$.
So, assuming that $s << n$ then 
\[
\| \Mb_i \| = O_\PP(\frac{s^{1/4} \sqrt{\log (1/\delta)}}{n^{3/4}})
\]
\end{proof}

\begin{lemma}
  \label{lem:residBdd}
  Consider both the empirical residuals, $\hat\etab$, and the true residuals, $\etab$.
Then
\[
\max_{i \in [n]} |\hat{\eta}_i^2 - \eta_i^2| = o_\PP \left(  \frac{1 }{n^{1/2 + \gamma}} \right)
\]
for any $\gamma \in [0,(1 - \alpha_S)/2)$.
\end{lemma}

\begin{proof}

First let us expand the terms in question:
\[
\hat{\eta}_i^2 - \eta_i^2 = [(I - \Mb)\eta]_i^2 -\eta_i^2 = [\eta - \Mb \eta]_i^2 - \eta_i^2= (\Mb \eta)_i^2 - 2 \eta_i (\Mb \eta)_i^2 
\]
Now we take a closer look at the right hand side,
\[
(\Mb \eta)_i^2 - 2 \eta_i (\Mb \eta)_i^2 = (\sum_i^{[n]} M_{i,j} \eta_j )^2 - 2 \eta_i \sum_j^{[n]} M_{i,j} \eta_j
\]
Notice that the true residuals $\eta$ are IID sub-Gaussian with parameter at most $\bar{\sigma}$.
Hence, with probability $1 - \delta$
\[
|\sum_i^{[n]} M_{i,j} \eta_j| \le \| \Mb_i \| \bar{\sigma} \sqrt{ 2 \log(2/\delta)}
\]
Hence, we find that 
\[
|(\Mb \eta)_i^2 - 2 \eta_i (\Mb \eta)_i^2|  = O( \| \Mb_i \| \bar{\sigma}^2 \log{1/\delta} )
\]
Below we show that there is a constant $C > 0$ such that 
\[
\bar{\sigma}^2 = O(\exp (\sqrt{ C \| \theta \|^2 \log(2n/\delta) }))
\]
with probability at least $1-\delta$.
Hence,
\[
\max_{i \in [n]} |\hat{\eta}_i^2 - \eta_i^2| = O(\frac{s^{1/4} (\log (n/\delta))^2 }{n^{3/4}}  \exp (\sqrt{ C \| \theta \|^2 \log(2n/\delta) })
\]
with probability $1 - \delta$.  Because $s = n^\alpha$ we have our result then this reduces to,
\[
\max_{i \in [n]} |\hat{\eta}_i^2 - \eta_i^2| = O(\frac{ (\log (n/\delta))^2 }{n^{(3- \alpha)/4}}  \exp (\sqrt{ C \| \theta \|^2 \log(2n/\delta) }) )
\]
Because both $\log (n/\delta))^2$ and $\exp (\sqrt{ C \| \theta \|^2 \log(2n/\delta) }) $ are subpolynomial in $n$, so, 
\[
\max_{i \in [n]} |\hat{\eta}_i^2 - \eta_i^2| = O_\PP (\frac{1}{n^{1/2 + \gamma}})
\]
where $\gamma > 0$ may be arbitrarily close to but less than $(1 - \alpha_S)/2$.
\end{proof}

\begin{lemma}
\label{lem:Uconv}
  Consider the difference of the pseudo-likelihood gradient and the true likelihood gradient,
\[
\hat{\Ub}_n - \Ub_n = \sum_i^{[n]} \frac{\hat{\eta}^2_i - \eta^2_i}{\sqrt{n}} e^{-x_i \theta} \xb_i
\]
If $\alpha_T < (1 - \alpha_S) / 2$ then
\[
\| \hat{\Ub}_n - \Ub_n \| = o_\PP (1)
\]
\end{lemma}

\begin{proof}

By the above lemma, 
\[
\max_{i \in [n]} \left| \sqrt{n}(\hat{\eta}^2_i - \eta^2_i) \right| = o_\PP (\frac{1}{n^{\gamma}})
\]
Moreover, we know that $\max_{i \in [n]} e^{-x_i \theta} X_i = O(\sqrt{t} \phi(n))$ where $\phi(n)$ is sub-polynomial in $n$.
Because $\sqrt{t} = n^{\alpha_T / 2}$ and $\alpha_T < 2 \gamma$ then
\[
\max_{i \in [n]} \left|  \sqrt{n}(\hat{\eta}^2_i - \eta^2_i) e^{-\xb_i \theta} X_i \right| = o_\PP(1)
\]
Thus the average of the summands is $o_\PP(1)$ and we have our result.
\end{proof}

\begin{lemma}
\label{lem:Vconv}
  Consider the difference of the pseudo-likelihood Hessian and the true Hessian,
\[
\hat{\Vb}_n - \Vb_n  = \frac{1}{n} \sum_{i}^{[n]} (\hat{\eta}^2_i - \eta^2_i) e^{-\xb_i' \thetab_T} \xb_i \xb_i'
\]
If $\alpha_T < (3 - \alpha_S) / 4$ then
\[
\opnorm{\hat{\Vb}_n -\Vb_n }{} = o_\PP (1)
\]
\end{lemma}

\begin{proof}
This proof follows in the same way as Lemma \ref{lem:Uconv} mutatis mutandis.
\end{proof}

\begin{lemma}
\label{lem:pmleExpan}
Let $\hat\thetab_T$ be the pseudo likelihood maximizer and $|T| = n^{\alpha_T}$ with $\alpha_T \in (0,\frac{1}{3})$.  And define the gradient and hessian,
\[
\hat{\Ub}_n = \frac{1}{\sqrt{n}} \sum_{i=1}^n \xb_{i,T} - \hat{\eta}_i^2 e^{-\xb_{i,T}' \thetab_T} \xb_{i,T}
\]
\[
\hat{\Vb}_n = \frac{1}{n} \sum_{i=1}^n \hat{\eta}_i^2 e^{-\xb_{i,T}' \thetab_T} \xb_{i,T} \xb_{i,T}'
\]
then we have that,
\[
\norm{\sqrt{n} (\hat\thetab_T - \thetab_T) - \hat{\Vb}_n^{-1} \hat{\Ub}_n } = o_\PP(1)
\]
\end{lemma}

\begin{proof}
Denote the function $\hat{L}_n(\tilde\thetab_T) = \sum_{1 = 1}^n \ell (\thetab_T + \frac{\tilde\thetab_T}{\sqrt{n}} | \xb_i,\hat{\etab} )$ which has minimizer $\sqrt{n} (\hat\thetab_T - \thetab_T)$.
Consider the Taylor expansion for $\hat{L}_n$ about $\zero$, 
\[
\hat{L}_n(\tilde \thetab_T) = \hat{L}_n(\zero) + \hat{\Ub}_n^\top \tilde\thetab_T + \tilde\thetab_T^\top \hat{\Vb}_n \tilde\thetab_T + \hat R_n(r \tilde\thetab_T)
\]
for some $r \in [0,1]$.  In Lemmata \ref{lem:Uconv}, \ref{lem:Vconv} we devoted ourselves to understanding the first and second order terms in this expansion.  We now must concern ourselves with the remainder, third order term.
\begin{equation}
\label{eq:Rn}
\begin{aligned}
\hat R_n(\tilde \thetab_T) = \frac{-1}{n \sqrt{n}} \sum_i^{[n]} \hat{\eta}_i^2 e^{-\xb_i' \theta} (\xb_i' \tilde \thetab_T)^3\\
R_n(\tilde \thetab_T) = \frac{-1}{n \sqrt{n}} \sum_i^{[n]} \eta_i^2 e^{-\xb_i' \theta} (\xb_i' \tilde \thetab_T)^3
\end{aligned}
\end{equation}
Define $\thetab_T^* = \hat{V}_n^{-1} \hat{U}_n$ and $\Delta_n(\delta) = \sup_{|\tilde \thetab_T - \thetab_T^*| \le \delta} \hat R_n (r \thetab_T)$.
We control $\hat R_n$ by decomposing it as $\hat R_n = (\hat R_n - R_n) + R_n$.
\begin{equation}
  \label{eq:Rn_decomp}
  \begin{aligned}
    \norm{\hat R_n(\tilde \thetab_T)  - R_n(\tilde \thetab_T)} \le \frac{1}{n \sqrt{n}}\sum_i^{[n]} | \hat \eta_i^2 - \eta_i^2 | e^{-\xb_{i,T}' \thetab_T} \norm{ \xb_{i,T} } \norm{\tilde \thetab}^3 \\
    = O_\PP \left( \frac{\bar\sigma^3 \max_{i \in [n]} \norm{\xb_{i,T}}^3 \norm{\tilde \thetab}^3 }{n^{1 + \gamma}} \right)
  \end{aligned}
\end{equation}
We show in Lemma \ref{lem:event_1} that $\bar{\sigma}$ is sub-polynomial in $n$.
Moreover, by the sub-Gaussianity of $\xb_{i,T}$, $\max_{i \in [n]} \norm{\xb_{i,T}}^3 = O_\PP(|T|^{3/2})$ modulo logarithmic terms.
If we assume that $\delta < 1$ then $\norm{ \tilde\thetab }^3 < 8 \norm{\hat\Vb_n^{-1} \hat\Ub_n }^3 $.
We have shown in the previous lemmata that $\norm{ \hat\Vb_n^{-1} \hat\Ub_n  - \Vb_n^{-1} \Ub_n } = o_\PP(1)$ under our assumptions.
By standard maximum likelihood estimation analysis $\norm{\Vb_n^{-1} \Ub_n } = O_\PP (\sqrt{|T|}) = O_\PP(n^{\alpha_T/2})$.
Hence, the RHS of eq.~\eqref{eq:Rn_decomp} is of the order $O_\PP( n^{3 \alpha_T - 1 - \gamma} )$ modulo sub-polynomial terms.
Due to the assumption that $\alpha_T < 1/3$ we have that $\norm{\hat R_n - R_n} = o_\PP(1)$.
Notice that this convergence is uniform over $\tilde \thetab_T$ because of the specific form of eq.~\ref{eq:Rn_decomp}. 

What remains to be shown is that $R_n(\tilde \thetab_T) = o_\PP(1)$ uniformly over $|\tilde \thetab_T - \thetab_T^*| \le \delta$ for some small $\delta$.
By identical arguments to those above we see that $\norm{ e^{-\xb_{i,T}' \thetab_T} (\xb_{i,T}' \tilde \thetab)^3} = O_\PP(\norm{\xb_{i,T}} n^{3 \alpha_T/2})$ uniformly over $i$ modulo logarithmic terms.
So the function $f_i(\eta_i,\xb_{i,T}| \tilde \thetab) = \eta_i^2 e^{-\xb_{i,T}' \thetab_T} (\xb_{i,T}' \tilde \thetab)^3$ may be bounded uniformly over a domain of radius $O_\PP(n^{3 \alpha_T/2})$.
$f_i(\eta_i,\xb_{i,T}| \tilde \thetab) \le \norm{\eta_i^2 \xb_{i,T} n^{3 \alpha_T/2}}$ and $\EE \eta_i^2 \norm{\xb_{i,T}}^3 n^{3 \alpha_T/2} = O_\PP(\bar\sigma^2 n^{3 \alpha_T})$.
Finally, we may invoke the uniform law of large numbers yielding $\frac{1}{n\sqrt{n}} R_n(\tilde \thetab_T) \Pconv 0$ if $\alpha_T < 1/3$ as $\bar\sigma$ is sub-polynomial in $n$.

We have shown that $\Delta(\delta) \Pconv 0$ if $\delta$ is decreasing.  
By Lemma 2 in \cite{pollard93}, we know that if $\Lambda_{\min} (\hat{V}_n)$ is bounded below in probability and $\Delta(\delta) \Pconv 0$ uniformly then the minimizer of $\hat L_n$ converges in probability to the minimizer of its quadratic approximation, viz.~$\norm{\sqrt{n} (\hat\thetab_T - \thetab_T) - \hat{\Vb}_n^{-1} \hat{\Ub}_n } = o_\PP(1)$.
We have established the latter condition while the former is a direct result of Lemma \ref{lem:Vconv}.
\end{proof}

Now, we can form the decomposition,
\begin{equation}
  \label{eq:cross1}
  \begin{aligned}
    \norm{\sqrt{n} (\hat\thetab_T - \thetab_T) - \Vb_n^{-1} \Ub_n} \le \norm{\sqrt{n} (\hat\thetab_T - \thetab_T) - \hat\Vb_n^{-1} \hat\Ub_n} + \norm{\Vb_n^{-1} \Ub_n - \hat\Vb_n^{-1} \hat\Ub_n} \\
  \le \norm{\Vb_n^{-1} \Ub_n - \Vb_n^{-1} \hat\Ub_n} + \norm{\Vb_n^{-1} \hat\Ub_n - \hat\Vb_n^{-1} \hat\Ub_n} + o_\PP(1)\\
  \le \opnorm{\Vb_n^{-1}}{} \norm{\Ub_n - \hat\Ub_n} + \opnorm{\Vb_n^{-1} - \hat\Vb_n^{-1}}{} \norm{\hat\Ub_n} + o_\PP(1)
 \end{aligned}
\end{equation}
All of these terms are decaying in probability because, $\opnorm{\Vb_n^{-1} - \hat\Vb_n^{-1}}{} \le \opnorm{\Vb_n - \hat\Vb_n}{} \opnorm{\Vb_n^{-1}}{} \opnorm{\hat\Vb_n^{-1}}{}$.
Combining these results we find that
\begin{equation}
\label{eq:strong_coupling}
\norm{\sqrt{n} (\hat\thetab_T - \thetab_T) - \Vb_n^{-1} \Ub_n} = o_\PP(1)
\end{equation}
We are now in a position to establish the oracle property of
\[
\norm{n^{\frac{1 - \alpha_T}{2}} (\hat\thetab_T - \thetab_T)} = O_\PP(1)
\]
Specifically, by standard MLE analysis we know that for a fixed coordinate in $j \in T$, $(\Vb_n^{-1} \Ub_n)_j = O_\PP(1)$.
Thus, $\norm{\sqrt{n} (\hat\thetab_T - \thetab_T)} \le \norm{\sqrt{n} (\hat\thetab_T - \thetab_T) - \Vb_n^{-1} \Ub_n} + \norm{\Vb_n^{-1} \Ub_n} = O_\PP(n^{\alpha_T/2})$.
Finally, we establish that the estimator $\hat{\thetab_T} = (\hat{\thetab_T}'_T,\mathbf{0}'_{S^C})'$ is a strict local minimizer of the program \eqref{eq:stage2}.

The first order conditions for a local optima of the SCAD penalized
likelihood of $\sigma$ are,
\begin{equation}
\label{eq:1o1}
\sum_j^{[n]} (1 - \hat{\eta}_j^2 e^{-\xb_j \hat{\theta}}) \xb_{j,i}^\top - n \bar{\rho}'_{\lambda_n}(\hat{\theta}_i) = 0, \textrm{ if } \theta_i \ne 0
\end{equation}
\begin{equation}
\label{eq:1o2}
|\sum_j^{[n]} (1 - \hat{\eta}_j^2 e^{-\xb_j \hat{\theta}}) \xb_{j,i}^\top| < n \rho_{\lambda_n}'(0+), \textrm{ if } \theta_i = 0
\end{equation}
where $\bar{\rho}'_{\lambda_n}(a) = \sgn(a)\rho'_{\lambda_n}(a)$.
The second order condition is given by,
\begin{equation}
\label{eq:1o3}
\Lambda_{\min} \left( \frac{1}{n} \Xb_1 \textrm{diag}\{ \exp(-\xb_i \theta_1)\} \Xb_1^\top \right) > \kappa_{\lambda_n} (\theta_1)
\end{equation}
where $\kappa_{\lambda_n}(\theta) = \max_{j \in T} \{ -\rho_{\lambda_n}''(|\theta_j|) \}$.

By the previous findings we have that uniformly over $j$, $|\hat\thetab_j| > |\thetab_j| - O_\PP(n^{\frac{1 - \alpha_T}{2}})$ so if $\lambda_T = \omega(n^{\frac{1 - \alpha_T}{2}})$ and $\thetab_j > C \lambda_T$ for a $C > 0$ specific to $\rho$, eq.~\eqref{eq:1o1} follows.
Moreover, by assumption 5, and similar arguments eq.~\eqref{eq:1o3} holds.  What remains to be shown is eq.~\eqref{eq:1o2}.  Recall the decomposition of $\xb_i$ in Lemma \ref{lem:Ej_subgauss}. 
\begin{equation}
  \label{eq:cross1}
  \begin{aligned}
    \norm{\sum_i^{[n]} (1 - \hat \eta^2_i e^{-\xb_i' \hat\thetab}) \xb_{i,T^C}}_\infty = \norm{ \sum_i^{[n]} (1 - \hat \eta^2_i e^{-\xb_i' \hat\thetab}) ( \xb_{i,T^C} \tau_i + \Eb_i)}_\infty\\
    = \norm{ \sum_i^{[n]} (1 - \hat \eta^2_i e^{-\xb_i' \hat\thetab}) \Eb_i}_\infty \le \norm{ \sum_i^{[n]} (\hat\eta^2_i - \eta_i^2) e^{-\xb_i' \hat\thetab} \Eb_i  }_\infty + \norm{\sum_i^{[n]} (1 - \eta^2_i e^{-\xb_i' \hat\thetab}) \Eb_i}_\infty
  \end{aligned}
\end{equation}
We will soon show that $e^{-\xb_i' (\hat\thetab - \thetab)} = o_\PP(1)$ and know by Lemma \ref{lem:event_1} that $e^{-\xb_i' \thetab }$ is $O_\PP(\phi(n))$ where $\phi(n)$ is subpolynomial in $n$.
Hence,
\begin{equation}
  \label{eq:cross2}
  \begin{aligned}
   \norm{ \sum_i^{[n]} (\hat\eta^2_i - \eta_i^2) e^{-\xb_i' \hat\thetab} \Eb_i  }_\infty \le \sqrt{ \sum_i^{[n]} (|\hat\eta^2_i - \eta_i^2| e^{-\xb_i' \hat\thetab} )^2 \sum_i^{[n]} \sum_i^{[n]} \norm{\Eb_i}_\infty^2 }\\
   = \sqrt{ O_\PP( \frac{\phi(n)}{n^{2 \gamma}} ) O_\PP(n \log |T^C|) } = O_\PP(n^{1/2 - \gamma} \sqrt{\phi(n) \log (p)})
  \end{aligned}
\end{equation}
This follows from the fact that $E_{i,j}^2$ is sub-exponential, hence, $\max_j E_{i,j}^2 = O(\log|T^C|)$.
We now show that $e^{-\xb_i' (\hat\thetab - \thetab)} = o_\PP(1)$ from the subgaussianity of $\norm{\xb_{i,T}}$.  Specifically, because we know that $\hat \thetab - \thetab = O_\PP (n^{(\alpha_T - 1)/2} )$ we will prove this uniformly over a neighborhood of radius $O(n^{(\alpha_T - 1)/2} )$.
Notice that over this neighborhood, $e^{-\xb_i' (\hat\thetab - \thetab)} \le \exp( r \norm{\xb_{i,T}}/n^{(1-\alpha_T)/2} )$ for some $r>0$.
\begin{equation}
  \label{eq:cross3}
  \begin{aligned}
    \PP \{ r \norm{\xb_{i,T}} > \nu n^{\alpha/2}\} \le C e^{-c\nu} \Rightarrow  \PP \{ r \norm{\xb_{i,T}} > \nu n^{\alpha_T/2}\} \le C e^{-c\nu} \\
    \Rightarrow \PP \{ \frac{r \norm{\xb_{i,T}}}{n^{(1-\alpha_T)/2}} > \nu\} \le C \exp(-c\nu(n^{1/2 - \alpha_T})) \\
    \Rightarrow \PP \{ \exp(\frac{r \norm{\xb_{i,T}}}{n^{(1-\alpha_T)/2}}) > \nu\} \le C \frac{1}{\nu^{c(n^{1/2 - \alpha_T})}}
 \end{aligned}
\end{equation}
In a similar way we may bound the square of $\exp(\frac{r \norm{\xb_{i,T}}}{n^{(1-\alpha_T)/2}})$ in probability.
Now in a uniform bounding technique similar to eq.~\eqref{eq:Rn_decomp} we can show that the second term in the RHS of eq.~\eqref{eq:cross2} is $O_\PP (n^{\alpha_T - 1/2} \log(p))$.
Define the variables $\xi_i = e^{-\xb_{i,T}'(\hat\thetab - \thetab)} - 1$ and $b_n = o(n^{1/2 - \alpha_T})$.
We see that under the same neighborhood for $\hat \thetab - \thetab$,
\[
\PP \{ \xi_i > \frac{1}{c b_n}\} \le C [(1 + \frac{1}{c b_n})^{-c b_n}]^{n^{1/2 - \alpha_T} / b_n} = o(1)
\]
Hence, $\xi_i = o_\PP(n^{\alpha_T - 1/2} \log(n))$ and we may bound the remaining term. 
\begin{equation}
  \label{eq:cross4}
  \begin{aligned}
    |\sum_i^{[n]} (1 - \eta^2_i e^{-\xb_i' \hat\thetab}) Eb_{i,j}| = |\sum_{i}^{[n]} (1 - \epsilon_i^2e^{-\xb_{i,T}'(\hat\thetab - \thetab)}) E_{i,j} |\\
    \le |\sum_i^{[n]} (1 - \epsilon_i^2) E_{i,j}| + |\sum_i^{[n]} \epsilon_i^2 \xi_i E_{i,j}| \le O_\PP(\sqrt{n}) + O_\PP(n^{\alpha_T + 1/2} \log(p))
 \end{aligned}
\end{equation}
by the central limit theorem and Cauchy-Schwartz.  Similarly, we can bound this uniformly over $j$ and obtain an additional $\log(p)$ factor.
Hence, $\lambda_T \asymp n^{-1/2 + \alpha_T} \log(p) \log(n)$ for eq.~\eqref{eq:1o2} to hold.
Thus $\hat\thetab$ is the strict local minimizer of the SCAD penalized program.

The weak central limit theorem in \eqref{eq:variance_stage_2} is a direct application of the standard CLT to \eqref{eq:strong_coupling}.
Specifically, $\ab' \Vb_n \Ub_n \stackrel{D}{\longrightarrow} \Ncal(0, \zeta)$ because the Fisher information of the true likelihood with respect to $\thetab$ is $\Sigma_T$.
While $\ab' (\sqrt n (\hat \thetab - \thetab) - \Vb_n \Ub_n) \stackrel{\PP}{\longrightarrow} 0$ by \eqref{eq:strong_coupling} and the construction of $\hat\thetab$.

\subsection{Proof of Theorem~\ref{thm:wls_mean}}

The proof follows the same lines as the proof of
Theorem~\ref{thm:ols_mean}, however, there are some technical
challenges that arise from having only an estimate of the variance. 

We define $\hat \Wb = \diag(\exp(-\Xb\hat\thetab/2))$ and $\Wb =
\diag(\exp(-\Xb\thetab/2))$. Furthermore, we will use $\hat\Db_{SS} =
n^{-1}\Xb_S'\hat\Wb^2\Xb_S$ and $\Db_{SS} = n^{-1}\Xb_S'\Wb^2\Xb_S$.

We proceed to show that $\hat\betab_w = (\hat\betab_{w,S}',
\zero_{S^C}')'$ is a strict local minimizer of ~\eqref{eq:stage3}
where $\hat\betab_{w,S}' = n^{-1}\hat\Db_{SS}^{-1}\Xb_S'\hat\Wb^2\yb$,
by showing that $\hat\betab_w$ satisfies
\begin{equation}
  \label{eq:proof:w_scad_mean:cond_1}
  \Xb_S'\hat\Wb(\hat\Wb\yb - \hat\Wb\Xb\hat\betab) - n \sgn(\hat\betab_{w,S})\odot
  \pen'_{\lambda}(\hat\betab_{w,S}) = \zero,
\end{equation}
\begin{equation}
  \label{eq:proof:w_scad_mean:cond_2}
  \norm{\Xb'\uSc\hat\Wb(\hat\Wb\yb - \hat\Wb\Xb\hat\betab_w)}_\infty < n \pen'_\lambda(0+),
\end{equation}
and 
\begin{equation}
  \label{eq:proof:w_scad_mean:cond_3}
  \Lambda_{\min}(n^{-1}\Xb_S'\hat\Wb^2\Xb_S) > \max_{j \in S}\ \{-\pen''_\lambda(|\hat\beta_{w,j}|)\}.
\end{equation}

Recall the definition of the event $\Acal_1$ from
the proof of Theorem~\ref{thm:ols_mean},
\[
\Acal_1 = \left\{ \max_{i \in [n]} \exp(|\xb_i'\thetab|) 
   \leq \exp\left(
      \sqrt{K^2\Lambda_{\max}(\Sigmab_{TT})\norm{\thetab}_2^2\log(2n/\delta)}
    \right)
   \right\},
\]
with $\PP[\Acal_1] \geq 1 -\delta$. Also, recall that
\[
  \bar\sigma^2 = \exp\left(
      \sqrt{K^2\Lambda_{\max}(\Sigmab_{TT})\norm{\thetab}_2^2\log(2n/\delta)}
    \right).
\]

Next, we define the event
\[
\Acal_3  = \{ \max_{i \in [n]}\ \big|\xb_i'(\hat\thetab-\thetab)\big| 
\leq
K\norm{\hat\thetab-\thetab}_2\Lambda_{\max}^{1/2}(\Sigmab_{TT})\sqrt{\log(2n/\delta)} \}
\]
and note that $\orlnorm{\xb_i'(\hat\thetab-\thetab)}{2} \leq
K\norm{\hat\thetab-\thetab}_2\Lambda_{\max}^{1/2}(\Sigmab_{TT})$,
since under the assumptions $\hat\thetab$ has the same support as
$\thetab$. Using \eqref{eq:tail_bound_sub_gauss} and the union bound,
we have that $\PP[\Acal_3] \geq 1 - \delta$. 

We will also use the event
\[
\Acal_4 = \{
\max_{i \in [n]}\ (\hat\thetab - \thetab)'\xb_{i, T}\xb_{i,T}'(\hat\thetab - \thetab)
\leq K^2\Lambda_{\max}(\Sigmab_{TT})\norm{\hat\thetab-\thetab}_2^2
\log(2n/\delta).
\}
\]
Setting $u =
K\Lambda_{\max}^{1/2}(\Sigmab_{TT})\norm{\hat\thetab-\thetab}_2
\sqrt{\log(2n/\delta)}$ in \eqref{eq:norm_sub_gaussian} and applying
the union bound, we obtain that $\PP[\Acal_4] \geq 1 - \delta$.  

Finally, define the event
\[
\Acal_5 =
\{ \Lambda_{\max}(n^{-1}\Xb_T'\Xb_T) \leq 
   3\Lambda_{\max}(\Sigmab_{TT}) \}.
\]
Similar to the proof of Lemma~\ref{lem:event_2}, we have that
$\PP[\Acal_5] \geq 1 - \delta$ for $n$ large enough so that
$\epsilon(n, \delta) < 1$, with $\epsilon(n, \delta)$ defined in
Lemma~\ref{lemma:hsu_2}.

In the following analysis, we condition on the event $\Acal_1 \cap
\Acal_3 \cap \Acal_4$.

The following decomposition 
\begin{equation}
  \label{eq:proof:w_scad_mean:diff_beta}
\begin{aligned}
  \hat \betab_{w,S} - \betab_{w,S}
  & = n^{-1}(\hat\Db_{SS}^{-1} - \Db_{SS}^{-1})\Xb_S'(\hat\Wb^2-\Wb^2)\etab \\
  & \quad + n^{-1}(\hat\Db_{SS}^{-1} - \Db_{SS}^{-1})\Xb_S'\Wb^2\etab \\
  & \quad + n^{-1}\Db_{SS}^{-1}\Xb_S'(\hat\Wb^2 - \Wb^2)\etab \\
  & \quad + n^{-1}\Db_{SS}^{-1}\Xb_S'\Wb^2\etab
\end{aligned}
\end{equation}
will be useful for establishing a bound on $\norm{\hat \betab_{w,S} -
  \betab_{w,S}}_\infty$. We investigate each of the terms separately.

Let $\sigma^2(\xb, \thetab) = \exp(\xb_i'\thetab)$. First, we will
need to control the deviation of $\sigma^{-2}(\xb, \hat\thetab)$ from
$\sigma^{-2}(\xb, \thetab)$. Using the Taylor expansion
\begin{equation*}
\begin{aligned}
  &\frac{1}{\sigma^2(\xb_i, \hat\thetab)}\\
  &\quad = \frac{1}{\sigma^2(\xb_i, \thetab)} 
   - \frac{2}{\sigma^3(\xb_i, \thetab)}
       \frac{\partial \sigma}{\partial \thetab}(\xb_i, \thetab)
       (\hat\thetab - \thetab) \\
  &\qquad + (\hat\thetab - \thetab)'
     \frac{3[(\partial \sigma/\partial \thetab)(\xb,\xib)]'
            (\partial \sigma/\partial \thetab)(\xb,\xib)
           - \sigma(\xb_i,\xib)(\partial^2 \sigma/\partial \thetab^2)(\xb,\xib)}
        {\sigma^4(\xb_i,\xib)}
     (\hat\thetab - \thetab)  \\
  &\quad= \frac{1}{\sigma^2(\xb_i, \thetab)} 
   - \frac{2}{\sigma^3(\xb_i, \thetab)}
       \frac{\partial \sigma}{\partial \thetab}(\xb_i, \thetab)
       (\hat\thetab - \thetab)  
   + (\hat\thetab - \thetab)'{\cal T}_i(\hat\thetab - \thetab)
\end{aligned}
\end{equation*}
where $\xib$ satisfies $\norm{\xib-\thetab}_2 \leq \norm{\xib-\thetab}_2$.

On the event $\Acal_1 \cap \Acal_3$, we have
\begin{equation}
\label{eq:proof:w_scad_mean:sigma_1}
\begin{aligned}
\max_{i \in [n]}\ &\Big|\frac{2}{\sigma^3(\xb_i, \thetab)}
  \frac{\partial \sigma}{\partial \thetab}(\xb_i, \thetab)
  (\hat\thetab - \thetab)\Big| \\
&=
\max_{i \in [n]} |\exp(\xb_i'\thetab)\xb_i'(\hat\thetab-\thetab)|\\
&\leq \bar\sigma^2
\max_{i \in [n]} |\xb_i'(\hat\thetab-\thetab)| \\
&\leq 
C
\norm{\hat\thetab-\thetab}_2
\exp\left(C\norm{\thetab}_2 \sqrt{\log(2n/\delta)}\right)
\log(2n/\delta)
\end{aligned}
\end{equation}
where $C = K\Lambda_{\max}^{1/2}(\Sigmab_{TT})$. 

Basic
calculus gives us that ${\cal T}_i = [\tau_{ab}]_{ab}$ with
\[
{\tau}_{ab} = \frac{\exp(-\xb_i'\thetab)}{2}x_{ia}x_{ib}.
\]
On the event $\Acal_1 \cap \Acal_4$
\begin{equation}
\label{eq:proof:w_scad_mean:sigma_2}
\begin{aligned}
\max_{i \in [n]}\ (\hat\thetab - \thetab)'{\cal T}_i(\hat\thetab - \thetab)
\leq (\bar\sigma^2/2) K^2\Lambda_{\max}(\Sigmab_{TT})\norm{\hat\thetab-\thetab}_2^2
\log(2n/\delta).
\end{aligned}
\end{equation}

Combining \eqref{eq:proof:w_scad_mean:sigma_1} and
\eqref{eq:proof:w_scad_mean:sigma_2}, we have 
\begin{equation}
\label{eq:proof:w_scad_mean:W}
\opnorm{\hat\Wb^2 - \Wb^2}{2} 
\leq
\bar\sigma^2
C
\norm{\hat\thetab-\thetab}_2
(1 + C\norm{\hat\thetab-\thetab}_2/2)
\log(2n/\delta)
\end{equation}
where $C = K\Lambda_{\max}^{1/2}(\Sigmab_{TT})$. With this, we have
that
\begin{equation}
\label{eq:proof:w_scad_mean:D}
\begin{aligned}
  \opnorm{\hat\Db - \Db}{2} 
   & = \opnorm{n^{-1}\Xb_S'(\hat\Wb^{2}-\Wb^{2})\Xb_S}{2} \\
   & \leq \opnorm{n^{-1}\Xb_S'\Xb_S}{2}\opnorm{\hat\Wb^2 - \Wb^2}{2} \\
   & \leq \bar\sigma^2  3\Lambda_{\max}(\Sigmab_{SS})
     C\norm{\hat\thetab-\thetab}_2
     (1 + C\norm{\hat\thetab-\thetab}_2/2)
     \log(2n/\delta).
\end{aligned}
\end{equation}
From~\eqref{eq:proof:w_scad_mean:D} and under the assumptions of
Theorem, we have
\[
\Lambda_{\min}(\hat\Db_{SS}) 
\geq \Lambda_{\min}(\Db_{SS}) - \opnorm{\hat\Db_{SS} - \Db_{SS}}{2} 
\geq D_{\min}/2
\]
for sufficiently large $n$.  Combining the last two displays
\begin{equation}
\label{eq:proof:w_scad_mean:invD}
\begin{aligned}
\opnorm{\hat\Db_{SS}^{-1}-\Db_{SS}^{-1}}{2} &
 = \opnorm{\hat\Db_{SS}^{-1}(\Db_{SS}-\hat\Db_{SS})\Db_{SS}^{-1}}{2} \\
 & \leq\opnorm{\hat\Db_{SS}^{-1}}{2}\opnorm{\Db_{SS}-\hat\Db_{SS}}{2}\opnorm{\Db_{SS}^{-1}}{2} \\
 & \leq\frac{2}{D_{\min}^2}\opnorm{\Db_{SS}-\hat\Db_{SS}}{2}.
\end{aligned}
\end{equation}

We are now ready to bound each term in
\eqref{eq:proof:w_scad_mean:diff_beta}. For the first term, we have
\begin{equation}
\label{eq:proof:w_scad_mean:diff_beta:1}
\begin{aligned}
&n^{-1} \norm{(\hat\Db_{SS}^{-1}-\Db_{SS}^{-1})\Xb_S'(\hat\Wb^2-\Wb^2)\etab}_\infty \\
&\qquad \leq n^{-1}\norm{(\hat\Db_{SS}^{-1}-\Db_{SS}^{-1})\Xb_S'(\hat\Wb^2-\Wb^2)\etab}_2 \\
&\qquad \leq n^{-1/2}\opnorm{\hat\Db_{SS}^{-1}-\Db_{SS}^{-1}}{2}
   \opnorm{n^{-1}\Xb_S'\Xb_S}{2}^{1/2}\opnorm{\hat\Wb^2-\Wb^2}{2}\norm{\etab}_2 \\
&\qquad \leq
\frac{2}{D_{\min}^{2}\sqrt{n}}
   \opnorm{n^{-1}\Xb_S'\Xb_S}{2}^{3/2}
   \opnorm{\hat\Wb^2 - \Wb^2}{2}^2
   \norm{\etab}_2 \\
&\qquad \leq
\frac{2C^2 {\bar\sigma}^5 (3\Lambda_{\max}(\Sigmab_{SS}))^{3/2}}{D_{\min}^{2}\sqrt{n}}
\norm{\hat\thetab-\thetab}_2^2
(1 + C\norm{\hat\thetab-\thetab}_2/2)^2
\log^2(2n/\delta)\norm{\epsilonb}_2 \\
&\qquad \leq
\frac{2\sqrt{3}C^2 {\bar\sigma}^5 (3\Lambda_{\max}(\Sigmab_{SS}))^{3/2}}{D_{\min}^{2}}
\norm{\hat\thetab-\thetab}_2^2
(1 + C\norm{\hat\thetab-\thetab}_2/2)^2
\log^2(2n/\delta)
\end{aligned}
\end{equation}
with probability at least $1 - \exp(-2n)$. The last inequality follows
from Lemma~\ref{lem:norm_sub_gaussian} with $u = \sqrt{3}$.

Similarly, the second term in \eqref{eq:proof:w_scad_mean:diff_beta}
yields
\begin{equation}
\label{eq:proof:w_scad_mean:diff_beta:2}
\begin{aligned}
&n^{-1}\norm{(\hat\Db_{SS}^{-1} - \Db_{SS}^{-1})\Xb_S'\Wb^2\etab}_\infty\\
&\qquad \leq n^{-1}\norm{(\hat\Db_{SS}^{-1} - \Db_{SS}^{-1})\Xb_S'\Wb^2\etab}_2\\
&\qquad \leq \frac{2}{D_{\min}^2n}\opnorm{n^{-1}\Xb_S'\Xb_S}{2}\opnorm{\hat\Wb^2 - \Wb^2}{2}\norm{\Xb_S'\Wb\epsilonb}_2\\
&\qquad \leq \frac{2}{D_{\min}^2\sqrt{n}}
             \opnorm{n^{-1}\Xb_S'\Xb_S}{2}
             \opnorm{\hat\Wb^2 - \Wb^2}{2}
             \opnorm{n^{-1}\Xb_S'\Wb^2\Xb_S}{2}^{1/2}\norm{\epsilonb}_2\\
&\qquad \leq\frac{6\sqrt{3}C\bar\sigma^2\Lambda_{\max}(\Sigmab_{SS})D_{\max}^{1/2}}{D_{\min}^2}
            \norm{\hat\thetab-\thetab}_2
            (1 + C\norm{\hat\thetab-\thetab}_2/2)
            \log(2n/\delta)
\end{aligned}
\end{equation}
with probability at least $1 - \exp(-2n)$.

For the third term, we have
\begin{equation}
\label{eq:proof:w_scad_mean:diff_beta:3}
\begin{aligned}
&n^{-1}\norm{\Db_{SS}^{-1}\Xb_S'(\hat\Wb^2 - \Wb^2)\etab}_\infty \\
&\qquad\leq  n^{-1}\norm{\Db_{SS}^{-1}\Xb_S'(\hat\Wb^2 - \Wb^2)\etab}_2 \\
&\qquad\leq  \frac{1}{D_{\min}\sqrt{n}}\opnorm{n^{-1}\Xb_S'\Xb_S}{2}^{1/2}\opnorm{\hat\Wb^2 - \Wb^2}{2}\norm{\etab}_2 \\
&\qquad\leq
\frac{3C\bar\sigma^3\Lambda_{\max}^{1/2}(\Sigmab_{SS})}{D_{\min}}
\norm{\hat\thetab-\thetab}_2
(1 + C\norm{\hat\thetab-\thetab}_2/2)
\log(2n/\delta)
\end{aligned}
\end{equation}
with probability at least $1 - \exp(-2n)$.

Finally we deal with the forth term in
\eqref{eq:proof:w_scad_mean:diff_beta}. Proceeding as in the proof of
\eqref{eq:proof:scad_mean:inf_norm_beta}, we have that 
\begin{equation}
\label{eq:proof:w_scad_mean:diff_beta:4}
\begin{aligned}
n^{-1}\norm{\Db_{SS}^{-1}\Xb_S'\Wb^2\etab}_\infty \leq \sqrt{\frac{2\log(2s/\delta)}{D_{\min}n}}
\end{aligned}
\end{equation}
with probability $1 - \delta$. 

Combining \eqref{eq:proof:w_scad_mean:diff_beta:1},
\eqref{eq:proof:w_scad_mean:diff_beta:2},
\eqref{eq:proof:w_scad_mean:diff_beta:3} and
\eqref{eq:proof:w_scad_mean:diff_beta:4} we have that 
\begin{equation}
\label{eq:proof:w_scad_mean:diff_beta:final}
\norm{\hat \betab_{w,S} - \betab_{w,S}}_\infty =
\Ocal(\norm{\hat\thetab - \thetab}_2\exp(C\sqrt{\log n})\log(n)) \ll \lambda
\end{equation}
with probability at least $1 - \exp(-2n) - 5\delta$. This also shows
that $\min_{j \in S} |\hat\beta_j| \gg \lambda$. Therefore, we have
shown \eqref{eq:proof:w_scad_mean:cond_1} and
\eqref{eq:proof:w_scad_mean:cond_3}.

Using Lemma~\ref{lem:regression} and Lemma~\ref{lem:var_regression},
we write $\Xb_j \in \RR^n$ as $\Xb_j = \Xb_S \taub_S + \Eb_j$, $j \in
S^C$, with $\Eb_j$ having elements that are subgaussian with parameter
$K\sqrt{\Sigma_{j|S}}$.  Denote $\Pb^{\perp}_{w,S} = \Ib -
\hat\Wb\Xb_S(\Xb_S'\hat\Wb^2\Xb_S)^{-1}\Xb_S'\hat\Wb$ the projection
matrix. Then
\begin{equation*}
\begin{aligned}
n^{-1}\Xb_j'\hat\Wb(\hat\Wb\yb - \hat\Wb\Xb\hat\betab_w) 
& = n^{-1}(\Xb_S \taub_S + \Eb_j)'\hat\Wb\Pb_{w,S}^{\perp}\hat\Wb\yb \\
& = n^{-1}\Eb_j'\hat\Wb\Pb_{w,S}^{\perp}\hat\Wb\diag(\exp(\Xb\theta/2))\epsilonb\\
& = \Zb_j'\epsilonb.
\end{aligned}
\end{equation*}
Conditioned on $\Xb$, we have that
\begin{equation}
  \label{eq:proof:w_scad_mean:cond_2:bound}
  \PP[\max_{j \in S^C} |\Zb_j'\epsilonb| > t] \leq 2(p-s)
  \exp\left(-\frac{t^2}{\max_{j \in S^C} \norm{\Zb_j}_2^2}\right)
\end{equation}
using Lemma~\ref{lem:sum_sub_gaussian} and
\eqref{eq:tail_bound_sub_gauss} together with the union bound.

We proceed to bound $\max_{j \in S^C} \norm{\Zb_j}_2^2$.
Write 
\begin{equation}
\label{eq:proof:w_scad_mean:cond_2:decomposition}
\begin{aligned}
\Zb_j 
& = n^{-1}\Eb_j'(\hat\Wb-\Wb)\Pb_{w,S}^{\perp}(\hat\Wb-\Wb)\diag(\exp(\Xb\theta/2)) \\
& \ \ + n^{-1}\Eb_j'(\hat\Wb-\Wb)\Pb_{w,S}^{\perp} \\
& \ \ + n^{-1}\Eb_j'\Wb\Pb_{w,S}^{\perp}(\hat\Wb-\Wb)\diag(\exp(\Xb\theta/2)) \\
& \ \ + n^{-1}\Eb_j'\Wb\Pb_{w,S}^{\perp}. 
\end{aligned}
\end{equation}
Using \eqref{eq:proof:w_scad_mean:W}, we have that 
\begin{equation}
\label{eq:proof:w_scad_mean:cond_2:bound_z}
\begin{aligned}
\norm{\Zb_j}_2 
& \leq n^{-1} \norm{\Eb_j}_2 (
     \opnorm{\hat\Wb - \Wb}{2}^{2}\bar\sigma + 
     2\opnorm{\hat\Wb - \Wb}{2} +
     \bar\sigma ) \\
& \leq \sqrt{3K(\max_{j \in S}\Sigma_{j|S})}n^{-1/2}(
     \opnorm{\hat\Wb - \Wb}{2}^{2}\bar\sigma + 
     2\opnorm{\hat\Wb - \Wb}{2} +
     \bar\sigma )
\end{aligned}
\end{equation}
with probability $1 - (p-s)\exp(-2n)$ using
Lemma~\ref{lem:norm_sub_gaussian} with the union bound.
Therefore, we have that 
\[
\begin{aligned}
\norm{n^{-1}\Xb_j'\hat\Wb(\hat\Wb\yb - \hat\Wb\Xb\hat\betab_w)}_\infty
&\leq \sqrt{\max_{j \in S^C}\norm{\Zb_j}_2\log(2(p-s)/\delta)}\\
&=\Ocal(\sqrt{\bar\sigma\log(p-s)}n^{-1/2}).
\end{aligned}
\]
This concludes the proof of the first part.

Second part of Theorem follows similarly to the proof of
Theorem~\ref{thm:ols_mean}. We have already shown 
\[
  \hat \betab_{w,S} - \betab_{w,S} = 
     n^{-1}\Db_{SS}^{-1}\Xb_S'\Wb\epsilonb + o_p(1),
\]
and the rest follows as in Theorem~\ref{thm:ols_mean}.


\end{document}